\documentclass[letterpaper,twocolumn,10pt]{article}

\usepackage{usenix}
\usepackage{hyperref}       
\usepackage{url}            
\usepackage{booktabs}       
\usepackage{amsfonts}       
\usepackage{amsmath}
\usepackage{amsthm}
\usepackage{amssymb}
\usepackage{graphicx}
\usepackage{nicefrac}       
\usepackage{microtype}      
\usepackage{xcolor}         
\usepackage{etoolbox}
\usepackage{caption}
\usepackage{subcaption}
\usepackage{xspace}
\usepackage{enumitem}
\usepackage[ruled]{algorithm2e}
\usepackage{tabu}
\usepackage{multirow}
\usepackage{bm}
\usepackage{newtxtext,newtxmath}
\usepackage{balance}
\hypersetup{linkcolor={ForestGreen}}

\newtheorem{definition}{Definition}
\newtheorem{lemma}{Lemma}
\newtheorem{theorem}{Theorem}

\newtheorem{proposition}{Proposition}

\newcommand{\N}{\mathfrak{N}}
\newcommand{\agg}{\texttt{agg}\xspace}
\newcommand{\upd}{\texttt{upd}\xspace}
\newcommand{\AGG}{\textsc{Agg}\xspace}
\newcommand{\UPD}{\textsc{Upd}\xspace}
\newcommand{\etal}{\textit{et~al}.}
\newcommand{\norm}[2][2]{\lVert#2\rVert_{#1}}

\renewcommand{\paragraph}[1]{\vspace{1ex}\noindent{\bf #1}}
\def\name{{GAP}\xspace}

\AtBeginDocument{}%
\AtBeginDocument{}%
\AtBeginDocument{}%
\AtBeginDocument{}%
\AtBeginDocument{}%
\AtBeginDocument{}%
\AtBeginDocument{}%
\AtBeginDocument{}%

\newenvironment{revised}{\color{black}}{\color{black}}

\begin{document}

\date{}

\title{\Large \bf \name: Differentially Private Graph Neural Networks\\with Aggregation Perturbation}

\author{
{\rm Sina Sajadmanesh$^{1,2}$\quad Ali Shahin Shamsabadi$^3$\quad Aurélien Bellet$^4$\quad  Daniel Gatica-Perez$^{1,2}$}\\
$^1$Idiap Research Institute\quad $^2$EPFL\quad $^3$The Alan Turing Institute\quad $^4$Inria
} 

\maketitle

\begin{abstract}


\begin{revised}
In this paper, we study the problem of learning Graph Neural Networks (GNNs) with Differential Privacy (DP). We propose a novel differentially private \textbf{G}NN based on \textbf{A}ggregation \textbf{P}erturbation (GAP), which adds stochastic noise to the GNN's aggregation function to statistically obfuscate the presence of a single edge (edge-level privacy) or a single node and all its adjacent edges (node-level privacy). Tailored to the specifics of private learning, GAP's new architecture is composed of three separate modules: (i) the encoder module, where we learn private node embeddings without relying on the edge information; (ii) the aggregation module, where we compute noisy aggregated node embeddings based on the graph structure; and (iii) the classification module, where we train a neural network on the private aggregations for node classification without further querying the graph edges.
GAP's major advantage over previous approaches is that it can benefit from multi-hop neighborhood aggregations, and
guarantees both edge-level and node-level DP not only for training, but also at inference with no additional costs beyond the training's privacy budget. We analyze GAP's formal privacy guarantees using Rényi DP and conduct empirical experiments over three real-world graph datasets. We demonstrate that GAP offers significantly better accuracy-privacy trade-offs than state-of-the-art DP-GNN approaches and naive MLP-based baselines. 
\end{revised}
Our code is publicly available at \href{https://github.com/sisaman/GAP}{https://github.com/sisaman/GAP}.
\end{abstract}

\section{Introduction}\label{sec:intro}

Real-world datasets are often represented by graphs, such as social~\cite{qiu2018deepinf}, financial~\cite{wang2021review}, transportation~\cite{diao2019dynamic}, or biological~\cite{kearnes2016molecular} networks, modeling the relations (i.e., edges) between a collection of entities (i.e., nodes).
Graph Neural Networks (GNNs) have achieved state-of-the-art performance in learning over such relational data in various graph-based machine learning tasks, such as node classification, link prediction, and graph classification~\cite{kipf2017semi, zhang2018link, xu2018how}. 
Due to their superior performance, GNNs are now widely used in many applications, such as recommendation systems, credit issuing, traffic forecasting, drug discovery, and medical diagnosis \cite{ying2018graph, gkalelis2021objectgraphs, jiang2021graph, mohammadshahi-henderson-2020-graph, 9378164}. 

\paragraph{Privacy concerns.}
Despite their success, real-world deployments of GNNs raise privacy concerns when graphs contain personal data: for instance, social or financial networks involve sensitive information about individuals and their interactions.
Recent works~\cite{olatunji2021membership, he2021node, he2021stealing, wu2021linkteller} have extended the study of the privacy leakage of standard deep learning models to GNNs, showing the risk of information leakage regarding training data is even higher in GNNs, as they incorporate not only node features and labels but also the graph structure itself~\cite{duddu2020quantifying}. Consequently, GNNs are vulnerable to various privacy attacks, such as node membership inference~\cite{olatunji2021membership, he2021node} and edge stealing~\cite{he2021stealing, wu2021linkteller}.  For example, a GNN trained on a social network for friendship recommendation could reveal the existing relationships between the users via its predictions. 
As another example, a GNN trained on the social graph of COVID-19 patients can be used by government authorities to predict the spread of the disease, but an adversary may recover private information about the participating patients.

\begin{figure*}
    \centering
    \includegraphics[width=0.8\textwidth]{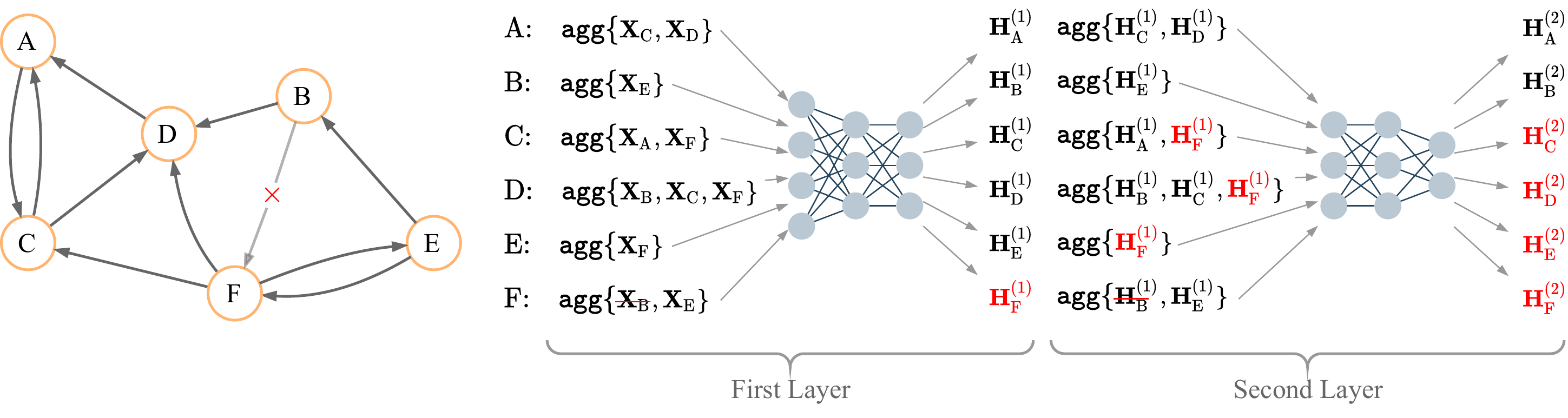}
    \caption{Schema of an unfolded 2-layer GNN taking an example graph as input. At each layer, every node aggregates its neighbors' embedding vectors (initially node features, e.g. $\mathbf{X}_\text{A}$ for node A), which is then updated using a neural net into a new vector (e.g., $\mathbf{H}_\text{A}$). Removing an arbitrary edge (here, the edge from node B to F) excludes the source node (B) from the aggregation set of the destination node (F). At the first layer, this will only alter the destination node's embedding, but this change is propagated to the neighboring nodes in the next layer. Node embeddings that are affected by the removal of edge (B,F) are indicated in red.}
    \label{fig:unfolded}
\end{figure*}

\paragraph{Problem and motivation.}
Motivated by these privacy concerns, we investigate the problem of designing privacy-preserving GNNs for private, sensitive graphs. 
Our goal is to 
protect the sensitive graph structure and other accompanying data using the framework of Differential Privacy (DP)~\cite{dwork2008differential}.
In the context of graphs, two different variants of DP have been defined: \emph{edge-level} and \emph{node-level} DP~\cite{raskhodnikova2016differentially}. Informally, an edge-level $\epsilon$-DP algorithm have roughly the same output (as measured by $\epsilon$) if one edge is removed from the input graph. This ensures that the algorithm's output does not reveal the existence of a particular edge in the graph. Correspondingly, node-level private algorithms conceal the presence of a particular node together with all its associated edges and attributes. Clearly, node-level DP is a stronger privacy definition, but it is harder to attain because it requires the algorithm's output distribution to hide much larger differences in the input graph.

\paragraph{Challenges.}
\begin{revised}
As GNNs utilize the structural information in the graph data, protecting data privacy in such models is more challenging than in standard ones. As shown in~\autoref{fig:unfolded}, one of these challenges is the interdependency between the node embeddings resulting from the GNN's data aggregation mechanism. Specifically, a $K$-layer GNN iteratively learns node embeddings by aggregating information from every node's $K$-hop neighborhood (i.e., from nodes that are at a distance at most $K$ in the graph). Hence, the embedding of a node is influenced not only by the node itself but also by all the nodes in its $K$-hop proximity. This fact voids the privacy guarantees of standard DP learning paradigms, such as DP-SGD~\cite{abadi2016deep}, as the training loss of GNNs can no longer be decomposed into individual samples. Furthermore,  the number of interdependent embeddings grows exponentially with $K$, hindering the ability of a DP solution to hide the output differences effectively. Therefore, how to get more representational power from higher-order GNN aggregations while ensuring DP is an important challenge to address. 

Another major challenge is to guarantee \emph{inference privacy}, i.e., preserving the privacy of graph data not only for training but also at inference time, when the trained GNN model is queried to make predictions for test nodes. Unlike conventional deep learning models, where the training data is not reused at inference time, the inference about any node in a $K$-layer GNN requires aggregating data from its $K$-hop neighborhood, which can reveal information about the neighboring nodes. Therefore, private graph data can still be leaked at inference time, even with privately trained model parameters. As a result, it is critical to ensure that both the training and inference stages of a GNN satisfy DP. This is illustrated in \autoref{fig:dp-comparison}.

\end{revised}

\begin{figure}[t]
    \centering
    \begin{subfigure}[b]{0.99\columnwidth}
        \centering
        \includegraphics[width=\textwidth]{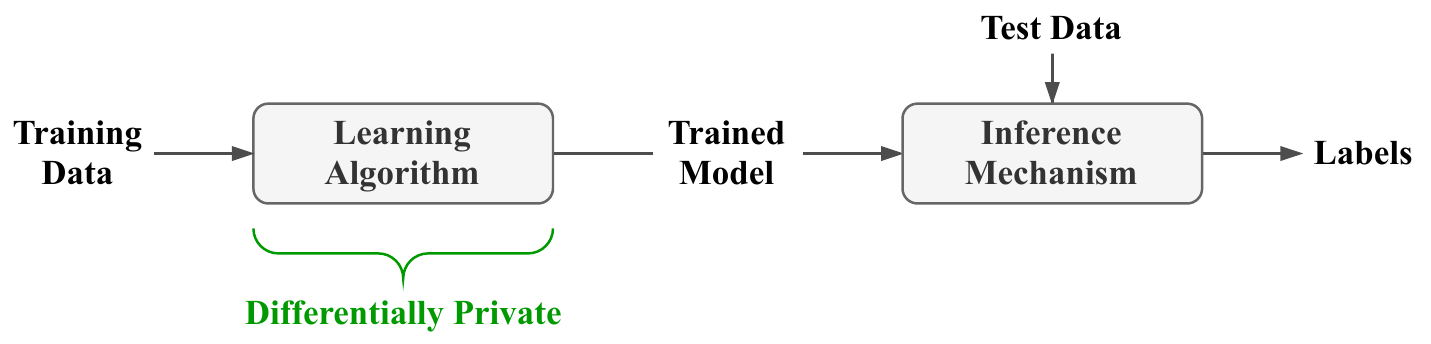}
        \caption{Learning standard DNNs with DP}
    \end{subfigure} 
    \\[0.5cm]
    \begin{subfigure}[b]{0.99\columnwidth}
        \centering
        \includegraphics[width=\textwidth]{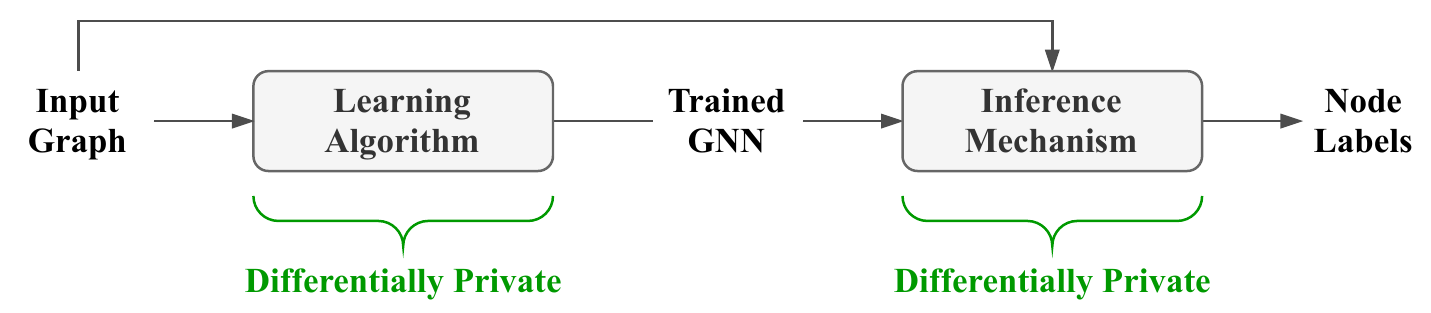}
        \caption{Learning GNNs with DP}
    \end{subfigure} 
    \caption{\begin{revised}Comparison of DP learning with (a) conventional deep neural networks, and (b) graph neural networks. Given the trained model, the inference mechanism of a DNN is independent of the training data, so a DP learning algorithm implies a DP inference mechanism as well. With GNNs however, graph data is queried again at inference time, so the inference step requires specific attention to be made differentially private.\end{revised}}
    \label{fig:dp-comparison}
\end{figure}

\paragraph{Our contributions.}
\begin{revised}To address the above challenges, we propose GAP, a privacy-preserving GNN model satisfying edge-level privacy, which is also extensible to node-level privacy if combined with standard private learning algorithms such as DP-SGD. As perturbing an edge in the input graph can practically be viewed as changing a sample in a node's neighborhood aggregation set, GAP preserves edge privacy via \emph{aggregation perturbation}: we add calibrated Gaussian noise to the output of the aggregation function, which can effectively hide the presence of a single edge (edge-level privacy) or a group of edges (node-level privacy).
To avoid accumulating privacy costs at every model update, we propose a custom GNN architecture (\autoref{fig:architecture}) comprising three individual components: (i) the encoder module, where we  pre-train an encoder to extract lower-dimensional node features without relying on the graph structure; (ii) the aggregation module, where we use aggregation perturbation to privately compute multi-hop aggregated node embeddings using the graph edges and the encoded features; and (iii) the classification module, where we train a neural network on the aggregated data for node classification without further querying the graph edges. 

Aggregation perturbation allows us to benefit from higher-order, multi-hop aggregations by composing individual noisy aggregations, yet the proposed architecture significantly reduces the privacy costs as the perturbed aggregations are computed once on lower-dimensional embeddings, and reused during training and inference. 
GAP also provides inference privacy, as the inference of any node relies on the perturbed aggregations, which hide information about neighboring nodes.
Due to reusing cached aggregations, the inference step does not incur additional privacy costs beyond that of training.
\end{revised}

\paragraph{Results.}
\begin{revised}
We analyze GAP's formal privacy guarantees using Rényi Differential Privacy~\cite{mironov2017renyi}, and empirically evaluate its accuracy-privacy performance on three medium to large-scale graph datasets, namely Facebook, Reddit, and Amazon. We demonstrate that \name's accuracy surpasses the competing baselines' at (very) low privacy budgets under both edge-level DP (e.g., $\epsilon\ge0.1$ on Reddit) and node-level DP (e.g., $\epsilon\ge1$ on Reddit), and observe that it always performs on par or better than a naive (privately trained) MLP model which does not utilize the graph's structural information.
\end{revised}

\begin{figure}[t]
    \centering
    \includegraphics[width=0.8\columnwidth]{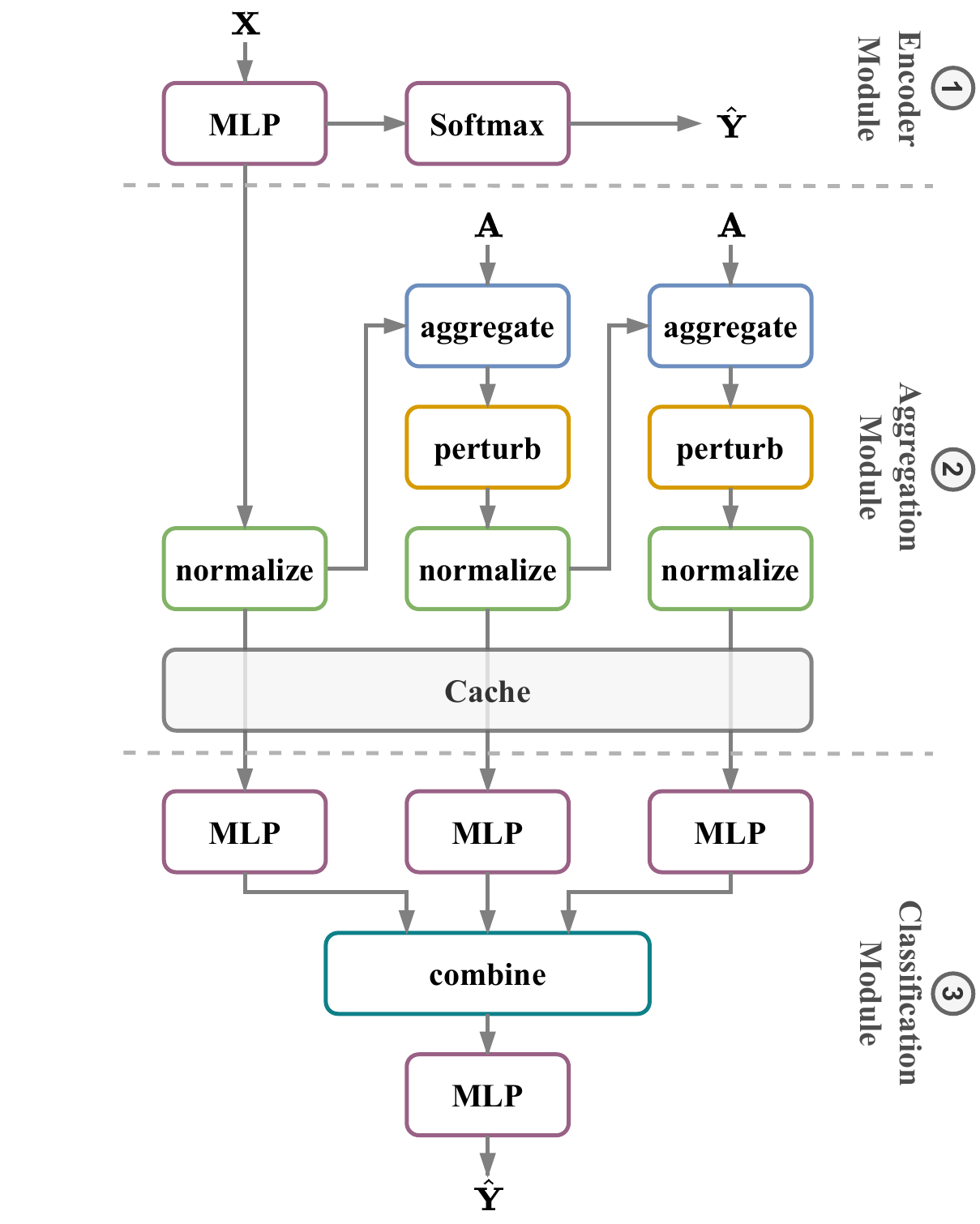}
    \caption{Overview of \name's architecture: (1) The encoder is trained using only node features ($\mathbf{X}$) and labels ($\mathbf{Y}$). (2) The encoded features are given to the aggregation module to compute private $K$-hop aggregations (here, $K=2$) using the graph's adjacency matrix ($\mathbf{A}$). (3) The classification module is trained over the private aggregations for label prediction.}
    \label{fig:architecture}
\end{figure}
\section{Related Work}\label{sec:related}

\paragraph{Graph neural networks.}
Deep learning on graphs has emerged in the past few years to tackle different kinds of graph-based learning tasks. A variety of GNN models and various architectures have been proposed, including Graph Convolutional Networks \cite{kipf2017semi}, Graph Attention Networks \cite{velivckovic2017graph}, GraphSAGE \cite{hamilton2017inductive}, Graph Isomorphism Networks \cite{xu2018how}, Jumping Knowledge Networks \cite{pmlr-v80-xu18c}, and Gated Graph Neural Networks \cite{li2015gated}. For the latest advances and trends in GNNs, we refer the reader to the available surveys~\cite{hamilton2017representation, wu2020comprehensive, 9039675, zhou2020graph, abadal2021computing}.

\paragraph{Privacy attacks on GNNs.}
Several recent works have investigated the possibility of performing privacy attacks against GNNs and quantified the privacy leakage of publicly released GNN models or node embeddings trained on private graph datasets. 
Zhang~\etal~\cite{277160} study the information leakage in graph embeddings and propose three different inference attacks against GNNs: inferring graph properties (such as number of nodes and edges), inferring whether a given subgraph is contained in the target graph, and graph reconstruction with similar statistics to the target graph.
He \etal~\cite{he2021stealing} propose a series of black-box link stealing attacks on GNN models, and show that an adversary can accurately infer a link between any pair of nodes in a graph used to train the GNN. 
Zhang~\etal~\cite{ijcai2021-516}
study the connection between model inversion risk and edge influence, and show that edges with greater influence are more likely to be inferred.
Wu~\etal~\cite{wu2021linkteller} also study the link stealing attack via influence analysis, and propose an effective attack against GNNs based on the node influence information.
The feasibility of the membership inference attack against GNNs has also been studied and several attacks with different threat models have been proposed in the literature~\cite{duddu2020quantifying, olatunji2021membership, he2021node, wypy2021miagnn}.
Overall, these works underline the privacy risks of GNNs trained on sensitive graph data and confirm the vulnerability of these models to various privacy attacks.

\paragraph{Differentially private GNNs.}
Recently, there have been attempts to use DP to provide formal privacy guarantees in various GNN learning settings.
Sajadmanesh and Gatica-Perez~\cite{sajadmanesh2021locally} propose a locally private GNN model by considering a distributed learning setting, where node features and labels are private but training the GNN is federated by a central server with access to graph edges. 
However, their method cannot be used in applications where the graph edges are private.
Wu~\etal~\cite{wu2021linkteller} propose an edge-level DP learning algorithm for GNNs by perturbing the input graph directly using either randomized response (called \textsc{EdgeRand}) or the Laplace mechanism (called \textsc{LapGraph}). Then, a GNN is trained over the resulting noisy graph. However, their method cannot be extended trivially to the node-level privacy setting.
Olatunji~\etal~\cite{olatunji2021releasing} consider a centralized learning setting and propose a node-level private GNN by adapting the framework of PATE~\cite{papernot2016semi}. They train the student GNN model using public graph data, which is privately labeled using the teacher GNN models trained exclusively for each query node. However, their dependence on public graph data restricts the applicability of their method.
Daigavane~\etal~\cite{daigavane2021node} also propose a node-level private approach for training 1-layer GNNs by extending the standard DP-SGD algorithm and privacy amplification by subsampling results to bounded-degree graph data. However, their approach fails to provide inference privacy and is limited to 1-layer GNNs and thus cannot leverage higher-order aggregations. 


\paragraph{Comparison with existing methods.}
To our best knowledge, GAP is the first approach providing both edge-level or node-level privacy guarantees based on the application requirements.
Unlike existing methods, our approach does not rely on public data, can leverage multi-hop aggregations beyond first-order neighbors, and guarantees inference privacy at no additional cost. In \autoref{sec:exp}, we also show that \name{} outperforms other baselines in terms of accuracy-privacy trade-off. 

\section{Background and Problem Formulation}\label{sec:basics}

\subsection{Graph Neural Networks}\label{sec:gnns}

GNNs aim to learn a representation for every node in the input graph by incorporating the initial node features and the graph structure (edges). The learned node representations, or embeddings, can then be used for the downstream machine learning task. 
In this paper, we focus on node classification, where the embeddings are used to predict the label of the graph nodes.
Node-wise prediction problems can be tackled in either transductive or inductive setting. In the transductive setting, both training and testing are performed on the same graph, but different nodes are used for training and testing. Conversely, in the inductive setting, training and testing are performed on different graphs. This is illustrated in~\autoref{fig:settings}.

\begin{figure}[t]
    \centering
    \begin{subfigure}[b]{0.4\columnwidth}
        \centering
        \includegraphics[width=\textwidth]{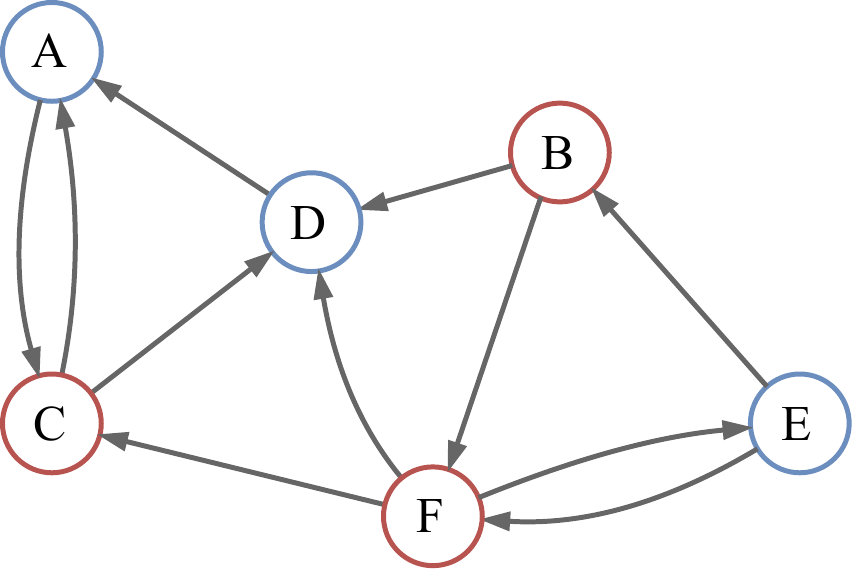}
        \caption{Transductive Learning}
		\label{fig:transductive}
    \end{subfigure} 
    \hfil
    \begin{subfigure}[b]{0.5\columnwidth}
        \centering
        \includegraphics[width=\textwidth]{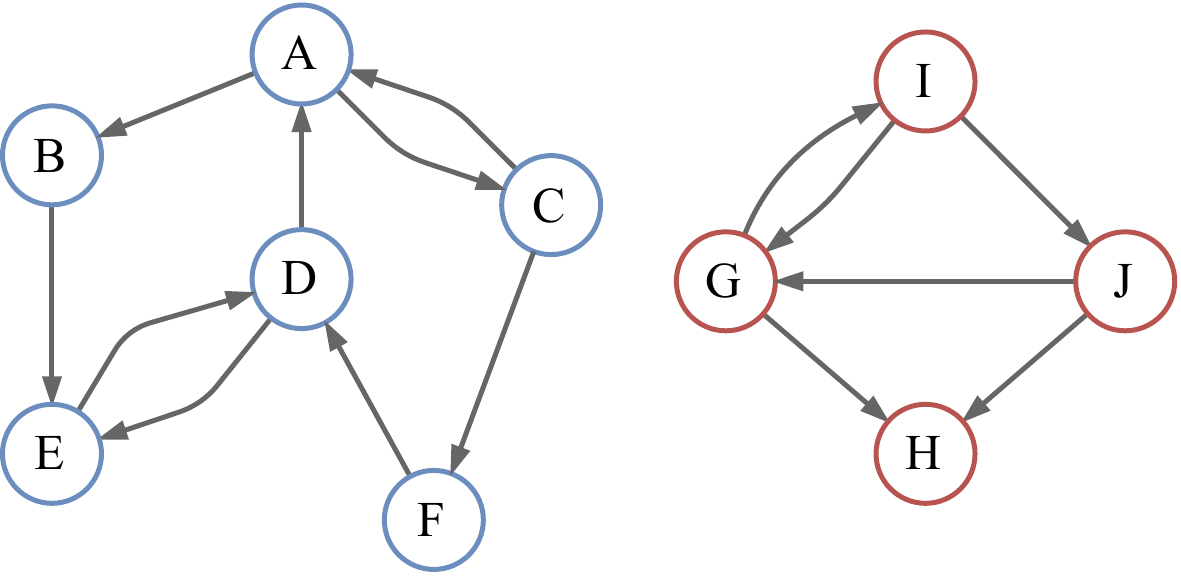}
        \caption{Inductive Learning}
		\label{fig:inductive}
    \end{subfigure} 
    \caption{\begin{revised}(a) Transductive learning: training and inference steps are conducted on the same graph, but different nodes are used for training and testing. Here, the blue nodes (A, D, and E) are used for training and the red nodes (B, C, and F) for inference. (b) Inductive learning: training and inference steps are performed on different graphs. Here, the left and right graphs are used for training and inference, respectively.\end{revised}}
    \label{fig:settings}
\end{figure}

Let $\mathcal{G}=(\mathcal{V}, \mathcal{E}, \mathbf{X}, \mathbf{Y})$ be an unweighted directed graph dataset consisting of sets of nodes  $\mathcal{V}$ and edges $\mathcal{E}$ represented by a binary adjacency matrix $\mathbf{A}\in\{0,1\}^{N\times N}$, where $N=\vert\mathcal{V}\vert$ denotes the number of nodes, and $\mathbf{A}_{i,j}=1$ if there is a directed edge $(i,j)\in\mathcal{E}$ from node $i$ to node $j$. Nodes are characterized by $d$-dimensional feature vectors stacked up in an $N\times d$ matrix $\mathbf{X}$, where $\mathbf{X}_v$ denotes the feature vector of the $v$-th node. $\mathbf{Y}\in\{0,1\}^{N\times C}$ represents the labels of the nodes, where $\mathbf{Y}_v$ is a $C$-dimensional one-hot vector denoting the label of the $v$-th node, and $C$ is the number of classes. Note that in the transductive learning setting, only a subset $\mathcal{V}_T\subset\mathcal{V}$ of the nodes is labeled, and thus $\mathbf{Y}_v$ is a zero vector for all $v\notin\mathcal{V}_T$.

A typical $K$-layer GNN consists of $K$ sequential graph convolution layers. Layer $i$ receives node embeddings from layer $i-1$ and outputs a new embedding for each node by aggregating the current embeddings of its adjacent neighbors followed by a learnable transformation, as defined below:
\begin{align*}
	\mathbf{H}^{(i)}_{v}     & = \upd \left(\agg \left(\{\mathbf{H}_u^{(i-1)}: \forall u \in \N_v\}\right); \mathbf{\Theta}^{(i)}\right),
\end{align*}
where $\N_v = \{u: \mathbf{A}_{u,v} \neq 0 \}$ denotes the set of adjacent nodes to node $v$ (i.e., nodes with outbound edges toward $v$), and $\mathbf{H}_{u}^{(i-1)}$ is the embedding of an adjacent node $u$ at layer $i-1$. $\agg(\cdot)$, is a (sub)differentiable, permutation invariant aggregator function, such as \textsc{Sum}, \textsc{Mean}, or \textsc{Max}. Finally, $\upd(\cdot)$ is a learnable function, such as a multi-layer perceptron (MLP), parameterized by $\mathbf{\Theta}^{(i)}$ that takes the aggregated vector and outputs the new embedding $\mathbf{H}^{(i)}_{v}$.
For convenience, we define the matrix-based version of $\agg(\cdot)$ and $\upd(\cdot)$ by stacking the corresponding vectors of all the nodes into a matrix as:
\begin{align*}
	&\AGG(\mathbf{H}, \mathbf{A}) = \left[\agg\left(\{\mathbf{H}_u: \forall u \in \N_v\}\right): \forall v\in\mathcal{V}\right]^T, \\
	&\UPD(\mathbf{M}; \mathbf{\Theta}) = \left[\upd\left(\mathbf{M}_{v}; \mathbf{\Theta}\right): \forall v\in\mathcal{V}\right]^T,
\end{align*}
where we omitted the layer indicator superscripts for simplicity. 
Initially, we have $\mathbf{H}^{(0)} = \mathbf{X}$ (i.e., node features) as the input to the GNN's first layer. The last layer generates an output embedding vector for each node, which can be used in different ways depending on the downstream task. For node classification, a softmax layer is applied to the final embeddings $\mathbf{H}^{(K)}$ to obtain the posterior class probabilities $\widehat{\mathbf{Y}}$. The illustration of a typical 3-layer GNN is depicted in \autoref{fig:gnn}. 

\begin{figure}[t]
    \centering
    \includegraphics[width=.8\columnwidth]{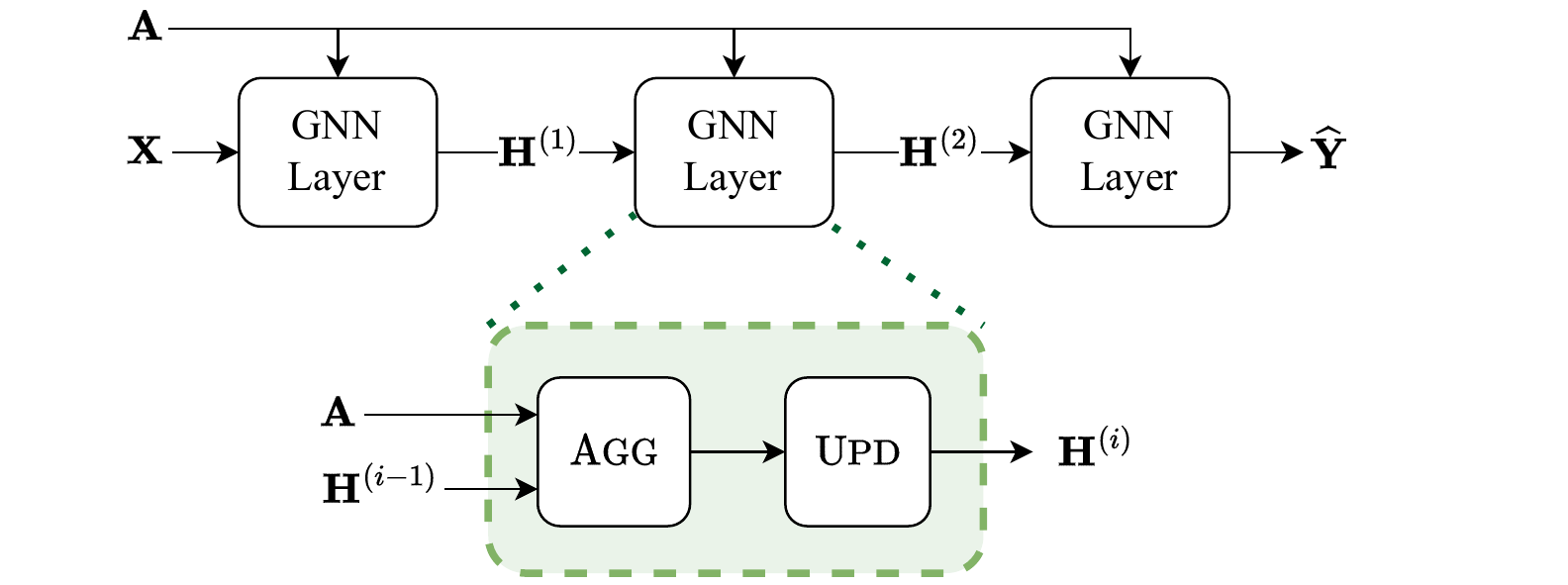}
    \caption{Typical 3-layer GNN for node classification. Each layer $i$ takes the adjacency matrix $\mathbf{A}$ and previous layer's node embedding matrix $\mathbf{H}^{(i-1)}$ (initially, node features $\mathbf{X}$), and outputs a new embedding matrix $\mathbf{H}^{(i)}$ (ultimately, predicted class labels $\widehat{\mathbf{Y}}$). Internally, the input embeddings $\mathbf{H}^{(i-1)}$ are aggregated based on the adjacency matrix $\mathbf{A}$, and then fed to a neural network (\UPD{}) to generate new embeddings $\mathbf{H}^{(i)}$.
	}
    \label{fig:gnn}
\end{figure}

\subsection{Differential Privacy}

Differential privacy (DP)~\cite{dwork2006calibrating} is the gold standard for formalizing the privacy guarantees of algorithms that process sensitive data. 
Informally, DP requires that the algorithm's output distribution be roughly the same regardless of the presence of an individual's data in the dataset. As such, an adversary having access to the data of all but the target individual cannot distinguish whether the target's record is among the input data. The formal definition of DP is as follows.
\begin{definition}[Differential Privacy~\cite{dwork2006calibrating}]\label{def:dp}
	Given $\epsilon > 0$ and $\delta > 0$, a randomized algorithm $\mathcal{A}$ satisfies $(\epsilon,\delta)$-differential privacy, if for all possible pairs of adjacent datasets $X$ and $X^\prime$ differing by at most one record, denoted as $X\sim X^\prime$, and for any possible set of outputs $S\subseteq Range(\mathcal{A})$, we have:
	\begin{equation*}
		\Pr[\mathcal{A}(X) \in S] \le e^\epsilon\Pr[\mathcal{A}(X^\prime) \in S] + \delta.
	\end{equation*}
\end{definition}
Here, the parameter $\epsilon$ is called the \emph{privacy budget} (or privacy cost) and is used to tune the privacy-utility trade-off of the algorithm: a lower privacy budget leads to stronger privacy guarantees but reduced utility. The parameter $\delta$ is informally treated as a failure probability, and is usually chosen to be very small. DP has the following important properties that help us design complex algorithms from simpler ones~\cite{dwork2008differential}:
\begin{itemize}[leftmargin=1em, noitemsep]
	\item \emph{Robustness to post-processing:} Any post-processing of the output of an $(\epsilon,\delta)$-DP algorithm remains $(\epsilon,\delta)$-DP.
	\item \emph{Sequential composition:} If an $(\epsilon,\delta)$-DP algorithm is applied $k$ times on the same data, the result is at most $(k\epsilon,k\delta)$-DP.
	\item \emph{Parallel composition:} Executing an $(\epsilon,\delta)$-DP algorithm on disjoint chunks of data yields an $(\epsilon,\delta)$-DP algorithm.
\end{itemize}

In this paper, we use an alternative definition of DP, called \emph{Rényi Differential Privacy} (RDP)~\cite{mironov2017renyi}, which allows obtaining tighter sequential composition results: 
\begin{definition}[Rényi Differential Privacy~\cite{mironov2017renyi}]
	A randomized algorithm $\mathcal{A}$ is $(\alpha,\epsilon)$-RDP for $\alpha > 1, \epsilon > 0$ if for every adjacent datasets $X\sim X^\prime$, we have $D_\alpha\left( \mathcal{A}(X) \Vert \mathcal{A}(X^\prime) \right) \le \epsilon$,
	where $D_\alpha(P\Vert Q)$ is the {Rényi divergence} of order $\alpha$ between probability distributions $P$ and $Q$ defined as:
	\begin{equation*}
		D_\alpha(P\Vert Q) = \frac{1}{\alpha-1}\log\mathbb{E}_{x\sim Q}\left[\frac{P(x)}{Q(x)}\right]^\alpha.
	\end{equation*}
\end{definition}

As RDP is a generalization of DP, it can be easily converted back to standard $(\epsilon,\delta)$-DP using the following proposition:
\begin{proposition}
\label{prop:rdptodp}
	If $\mathcal{A}$ is an $(\alpha,\epsilon)$-RDP algorithm, then it also satisfies $(\epsilon+\nicefrac{\log (1/\delta)}{\alpha-1},\delta)$-DP for any $\delta\in(0,1)$.
\end{proposition}

A basic method to achieve RDP is the \emph{Gaussian mechanism}, where Gaussian noise is added to the output of the algorithm we want to make private. Specifically, let $f: \mathcal{X}\rightarrow\mathbb{R}^d$ be the non-private algorithm taking a dataset as input and outputting a $d$-dimensional vector. Let the \emph{sensitivity} of $f$ be the maximum $L_2$ distance achievable when applying $f(\cdot)$ to adjacent datasets $X$ and $X^\prime$ as $\Delta_f = \max_{{X}\sim{X}^\prime} \left\Vert f({X}) - f({X}^\prime) \right\Vert_2$.
Then, adding Gaussian noise with variance $\sigma^2$ to $f$ as $\mathcal{A}(X) = f(X) + \mathcal{N}(\sigma^2\mathbb{I}_d)$,
with $\mathbb{I}_d$ being $d\times d$ identity matrix, yields an $(\alpha,\epsilon)$-RDP algorithm for all $\alpha>1$ with $\epsilon=\nicefrac{\Delta_f^2\alpha}{2\sigma^2}$ \cite{mironov2017renyi}.

\subsection{Problem Definition}

Let $\widehat{\mathbf{Y}} = \mathcal{F}(\mathbf{X},\mathbf{A}; \mathbf{\Theta})$ be a GNN-based node classification model with parameter set $\mathbf{\Theta}$ that takes node features $\mathbf{X}$ and the graph's adjacency matrix $\mathbf{A}$ as input, and outputs the corresponding predicted labels $\widehat{\mathbf{Y}}$.
To learn the model parameters $\mathbf{\Theta}$, we minimize a standard classification loss function (e.g., cross-entropy) with respect to $\mathbf{\Theta}$ as follows:
\begin{equation}\label{eq:train}
	\mathbf{\Theta}^\star = \arg\min_{\mathbf{\Theta}} \sum_{v\in\mathcal{V}_T} \ell(\widehat{\mathbf{Y}}_v, \mathbf{Y}_v),
\end{equation}
where $\ell(\cdot, \cdot)$ is the loss function, $\mathbf{Y}$ is the ground-truth labels, and $\mathcal{V}_T\subseteq\mathcal{V}$ is the set of labeled training nodes. After training, in the transductive setting, the learned GNN is used to infer the labels of unlabeled nodes in $\mathcal{G}$:
\begin{equation}\label{eq:inference}
	\widehat{\mathbf{Y}} = \mathcal{F}(\mathbf{X},\mathbf{A}; \mathbf{\Theta}^\star),
\end{equation}
Otherwise, in the inductive setting, a new graph dataset $\mathcal{G}_{test}$ is given to the learned GNN for label inference.

The goal of this paper is to preserve the privacy of graph datasets for both the training step (\autoref{eq:train}) and the inference step (\autoref{eq:inference}) using  differential privacy. Note that preserving privacy in the inference step is critical as the adjacency information is still used in this step for obtaining the predicted labels.

However, as graph datasets are different from standard tabular datasets due to the existence of links between data records, one needs to adapt the definition of DP to graphs. As the semantic interpretation of DP relies on the definition of adjacent datasets, we first define two different notions of adjacency in graphs, namely edge-level and node-level adjacent graph datasets~\cite{hay2009accurate}:
\begin{definition}[Edge-level adjacent graphs]
Two graphs $\mathcal{G}$ and $\mathcal{G}'$ are edge-level adjacent if one can be obtained by removing a single edge from the other. Therefore, $\mathcal{G}$ and $\mathcal{G}'$ differ by at most one edge.
\end{definition}

\begin{definition}[Node-level adjacent graphs]
Two graphs $\mathcal{G}$ and $\mathcal{G}'$ are node-level adjacent if one can be obtained by removing a single node (with its features, labels, and all attached edges) from the other. Therefore, $\mathcal{G}$ and $\mathcal{G}'$ differ by at most one node.
\end{definition}


Accordingly, the definition of edge-level and node-level DP follows from the above definitions: an algorithm $\mathcal{A}$ is edge-level (respectively, node-level) $(\epsilon, \delta)$-DP if for every two edge-level (respectively, node-level) adjacent graph datasets $\mathcal{G}$ and $\mathcal{G}^\prime$ and any set of outputs $S\subseteq Range(\mathcal{A})$, we have $\Pr[\mathcal{A}(\mathcal{G}) \in S] \le e^\epsilon\Pr[\mathcal{A}(\mathcal{G}^\prime) \in S] + \delta.$

Intuitively, edge-level DP protects edges (which could represent connections between people), while node-level DP protects nodes together with their adjacent edges (i.e., all information pertaining to an individual, including features, labels, and connections). 


\section{Proposed Method: \name}\label{sec:method}

In this section, we explain our proposed differentially private method, called GNN with Aggregation Perturbation (GAP), which guarantees both edge-level and node-level privacy for training and inference on sensitive graph data. 

\subsection{Overview}

\begin{revised}
As mentioned in~\autoref{sec:intro}, the two primary challenges in the design of private GNNs come from the use of higher-order aggregations and the need to ensure inference privacy. To tackle these challenges, we propose a new architecture for GAP, which is different from the conventional GNN architectures presented in~\autoref{sec:gnns}. The key distinction is that GAP decouples the graph-based aggregations from the neural network-based transformations, which is similar in spirit to the Inception model and scalable networks~\cite{szegedy2015going, wu2019simplifying, frasca2020sign}. As illustrated in~\autoref{fig:architecture}, GAP is composed of the following three components: 
\begin{enumerate}[leftmargin=2em,label=(\roman*)]
    \item \emph{Encoder Module (EM):} 
    This module encodes the input node features into a lower-dimensional representation without using the private graph structure. 
    
    \item \emph{Aggregation Module (AM):} This module takes the encoded low-dimensional node features and recursively computes private multi-hop aggregations using the \emph{aggregation perturbation} approach, i.e., by adding noise to the output of each aggregation step.
    
    \item \emph{Classification Module (CM):} This module takes the privately aggregated node features and predicts the corresponding labels without querying the edges any further.
\end{enumerate}

\paragraph{GAP's privacy mechanism.}
Our proposed mechanism for preserving the privacy of graph edges in AM is the \emph{aggregation perturbation} approach: we use the Gaussian mechanism to add stochastic noise to the output of the aggregation function proportional to its sensitivity. This approach is motivated by the fact that perturbing an edge in the input graph can practically be viewed as changing a sample in the neighborhood aggregation function of the edge's destination node. Therefore, by adding an appropriate amount of noise to the aggregation function, we can effectively hide the presence of a single edge, which ensures edge-level privacy, or a group of edges, which is necessary for node-level privacy. To fully guarantee node-level privacy, however, in addition to the edges, we need to also protect node features and labels, which is simply done by training EM and CM using standard DP learning algorithms such as DP-SGD. We discuss this point further in~\autoref{sec:priv}.

\paragraph{Challenges addressed.}
Our GAP method can benefit from multi-hop aggregations by composing individual noisy aggregation steps. As the sensitivity of a single-step aggregation is easily determined, AM applies the Gaussian mechanism immediately after each aggregation step, avoiding the growing interdependency between node embeddings.
GAP also provides inference privacy as the inference of a node relies on the aggregated data from its neighbors, which is privately computed by AM. As the subsequent CM only post-processes these private aggregations, GAP ensures inference-time privacy. This is explained in more details in~\autoref{sec:priv}.

In the rest of this section, we first discuss each of the GAP's components thoroughly and then describe the inference mechanism.

\end{revised}

\subsection{Encoder Module}\label{sec:em}

\begin{revised}
GAP uses a multi-layer perceptron (MLP) model as an encoder to transform the original node features into an intermediate representation given to AM. The main goal of this module is to reduce the dimensionality of AM's input, as the magnitude of the Gaussian noise injected into the aggregations grows with data dimensionality. Therefore, reducing the dimensionality helps achieve better aggregation utility under DP.

Note that in order to save the privacy budget spent in AM, we do not train the encoder end-to-end with CM. Instead, we attach a linear softmax layer to the encoder MLP for label prediction, and then pre-train this model separately using node features and labels. Specifically, we use the following model:
\end{revised}
\begin{equation}\label{eq:encmod}
	\widehat{\mathbf{Y}} = \text{softmax}\left(\text{MLP}_\text{enc}(\mathbf{X}; \mathbf{\Theta}_{\text{enc}})\cdot\mathbf{W}\right),
\end{equation}
where $\text{MLP}_\text{enc}$ is the encoder MLP with parameter set $\mathbf{\Theta}_{\text{enc}}$, $\mathbf{W}$ is the weight matrix of the linear softmax layer, $\mathbf{X}$ is the original node features, and $\widehat{\mathbf{Y}}$ is the corresponding posterior class probabilities. 
In order to train this model, we minimize the cross-entropy (or any other classification-related) loss function $\ell(\cdot, \cdot)$ with respect to the model parameters $\mathbf{\Theta} = \{ \mathbf{\Theta}_{\text{enc}}, \mathbf{W}\}$:
\begin{equation}\label{eq:loss}
	\mathbf{\Theta}^\star = \arg\min_{\mathbf{\Theta}} \sum_{v\in\mathcal{V}_T} \ell(\widehat{\mathbf{Y}}_v, \mathbf{Y}_v),
\end{equation}
where  $\mathbf{Y}$ is the ground-truth labels and $\mathcal{V}_T\subseteq\mathcal{V}$ is the set of training nodes. 
After pre-training, we use the encoder MLP to extract low-dimensional node features, $\mathbf{X}^{(0)}$, for AM:
\begin{equation}\label{eq:encoder}
	\mathbf{X}^{(0)} = \text{MLP}_\text{enc}(\mathbf{X}; \mathbf{\Theta}_{\text{enc}}^\star).
\end{equation}

\begin{revised}
\paragraph{Remark.} As will be discussed in~\autoref{sec:am}, this encoder pre-training approach significantly reduces the model's privacy costs as the private aggregations in AM no longer need to be updated with the encoder's parameters. Besides, compared to the original features, this approach provides better node features to AM as the encoded representations incorporate label information as well.

\end{revised}

\subsection{Aggregation Module}\label{sec:am}

The goal of AM is to privately release multi-hop aggregated node features using the aggregation perturbation method. \autoref{alg:pma} presents our mechanism, the Private Multi-hop Aggregation (PMA). It relies on the \textsc{Sum} aggregation function, which is simply equivalent to the multiplication of the adjacency matrix $\mathbf{A}$ by the input feature matrix $\mathbf{X}$, as $\AGG(\mathbf{X}, \mathbf{A}) = \mathbf{A}^T\cdot\mathbf{X}$.
The PMA mechanism takes $\check{\mathbf{X}}^{(0)}$, the row-normalized version of the encoder's extracted features as:
\begin{equation}\label{eq:normalized}
	\check{\mathbf{X}}^{(0)}_v = {\mathbf{X}^{(0)}_v}/{\norm{\mathbf{X}^{(0)}_v}},\quad\forall v\in\mathcal{V}.
\end{equation}
It then outputs a set of $K$ normalized, privately aggregated node features $\check{\mathbf{X}}^{(1)}$ to $\check{\mathbf{X}}^{(K)}$ corresponding to different hops from $1$ to $K$.
Specifically, given $\sigma>0$, the PMA mechanism performs the following steps to recursively compute and perturb the aggregations in $k\text{-th}$ hop from $(k-1)\text{-th}$:
\begin{enumerate}[leftmargin=1em, noitemsep]
	\item \textbf{Aggregation:} First, we compute $k$-th non-private aggregations using the normalized aggregations at step $k-1$:
	\begin{equation}
		\mathbf{X}^{(k)}=\mathbf{A}^T\cdot\check{\mathbf{X}}^{(k-1)}.
	\end{equation}
	\item \textbf{Perturbation:} Next, we perturb the aggregations using the Gaussian mechanism, i.e., by adding noise with variance $\sigma^2$ to every row of $\mathbf{X}^{(k)}$ independently:
	\begin{equation}
		\widetilde{\mathbf{X}}^{(k)}_v=\mathbf{X}^{(k)}_v + \mathcal{N}(\sigma^2\mathbb{I}), \quad\forall v\in\mathcal{V}.
	\end{equation}
	\item \textbf{Normalization:} Finally, it is essential to bound the effect of each feature vector on the subsequent aggregations. Therefore, we again row-normalize the private aggregated features, such that the L2-norm of each row is 1: 
	\begin{equation}
		\check{\mathbf{X}}^{(k)}_v = {\widetilde{\mathbf{X}}^{(k)}_v} / {||{\widetilde{\mathbf{X}}^{(k)}_v}||_2},\quad\forall v\in\mathcal{V}.
	\end{equation}
\end{enumerate}


\begin{revised}
\paragraph{Remark.} The recursive computation of aggregations in the PMA mechanism has one advantage: each aggregation step acts as a denoising mechanism, averaging out the DP noise added in the previous step (to some extent). Therefore, part of the injected noise is dampened by the PMA mechanism itself, leading to better aggregation utility. This noise-reducing effect of GNN aggregations is also observed in prior work~\cite{sajadmanesh2021locally}.

\paragraph{Effect of EM.} Note that EM plays a critical role in improving AM's privacy-utility trade-off: First, it increases the utility of noisy aggregations by reducing the dimensionality of AM's input, resulting in less noise added to the aggregations. Second, its pre-training strategy makes AM agnostic to model training, which remarkably reduces the total privacy costs as the PMA mechanism is called only once and its output is cached to be reused for entire training and inference. Technically, this implies that with $T$ training iterations, the Gaussian mechanism is composed only $K$ times, which would otherwise be $KT$ in the case of end-to-end training. Since $K$ is small ($1\le K\le5$) compared to $T$ (in the order of hundreds), this leads to a substantial reduction in the privacy budget.
\end{revised}

\begin{algorithm}[t]
    \caption{Private Multi-hop Aggregation}\label{alg:pma}
	\footnotesize
	\SetKwInOut{Input}{Input}
	\SetKwInOut{Output}{Output}
	\ResetInOut{Output}
	\LinesNumbered
	\Input{Graph $\mathcal{G}=(\mathcal{V}, \mathcal{E})$ with adjacency matrix $\mathbf{A}$; initial normalized features $\check{\mathbf{X}}^{(0)}$; max hop $K$; noise variance $\sigma^2$;}
	\Output{Private aggregated node feature matrices $\check{\mathbf{X}}^{(1)},\dots,\check{\mathbf{X}}^{(K)}$}
	\BlankLine
	\DontPrintSemicolon
	\For{$k\in\{1,\dots,K\}$} {
	    $\mathbf{X}^{(k)} \gets \mathbf{A}^T\cdot\check{\mathbf{X}}^{(k-1)}$ \tcp*{\color{RoyalBlue} aggregate}
		$\widetilde{\mathbf{X}}^{(k)} \gets \mathbf{X}^{(k)} + \mathcal{N}(\sigma^2\mathbb{I})$ \tcp*{\color{Dandelion} perturb\quad}
	    \For{$v\in\mathcal{V}$} {
    	    $\check{\mathbf{X}}^{(k)}_v \gets {\widetilde{\mathbf{X}}^{(k)}_v} / {||{\widetilde{\mathbf{X}}^{(k)}_v}||_2}$ \tcp*{\color{ForestGreen} normalize}
	    }
	}
	\Return  $\check{\mathbf{X}}^{(1)}, \dots, \check{\mathbf{X}}^{(K)}$\;
\end{algorithm}

\subsection{Classification Module}

Given the list of private aggregated features $\{\check{\mathbf{X}}^{(0)}, \dots,\check{\mathbf{X}}^{(K)}\}$ provided by AM, the goal of CM is to predict node labels without further relying on the graph edges. To this end, for each $k\in\{0,1,\dots,K\}$, we first obtain the $k$-hop representation $\mathbf{H}^{(k)}$ using a corresponding base MLP, denoted as $\text{MLP}_\text{base}^{(k)}$:
\begin{equation}\label{eq:mlpbase}
    \mathbf{H}^{(k)} = \text{MLP}_\text{base}^{(k)}(\check{\mathbf{X}}^{(k)}; \mathbf{\Theta}_\text{base}^{(k)}),
\end{equation}
where $\mathbf{\Theta}_\text{base}^{(k)}$ is the parameters of $\text{MLP}_\text{base}^{(k)}$.
Next, we combine these representations to get an integrated node embedding $\mathbf{H}$:
\begin{equation}\label{eq:comb}
    \mathbf{H} = \text{\textsc{Combine}}\left(\{\mathbf{H}^{(0)},\mathbf{H}^{(1)},\dots,\mathbf{H}^{(K)}\}; \mathbf{\Theta}_\text{comb}\right),
\end{equation}
where \textsc{Combine} is any differentiable combination strategy, with common choices being summation, concatenation, or attention, potentially with parameter set $\mathbf{\Theta}_\text{comb}$.
Finally, we feed the integrated representation into a head MLP, denoted as $\text{MLP}_\text{head}$, to get posterior class probabilities for the nodes:
\begin{equation}\label{eq:mlphead}
    \widehat{\mathbf{Y}} = \text{MLP}_\text{head}(\mathbf{H}; \mathbf{\Theta}_\text{head}),
\end{equation}
where $\mathbf{\Theta}_\text{head}$ denotes the parameters of $\text{MLP}_\text{head}$.
To train CM, we minimize a similar loss function as \autoref{eq:loss} but with respect to CM's parameters: $\mathbf{\Theta} = \{ \mathbf{\Theta}_\text{base}^{(0)}, \dots, \mathbf{\Theta}_\text{base}^{(K)}, \mathbf{\Theta}_\text{comb}, \mathbf{\Theta}_\text{head} \}$. The overall training procedure of GAP is presented in~\autoref{alg:gap}. 

\begin{revised}

\paragraph{Remark.} CM independently processes the information encoded in the graph-agnostic node features $\check{\mathbf{X}}^{(0)}$ and the private, graph-based aggregated features $\check{\mathbf{X}}^{(1)}$ to $\check{\mathbf{X}}^{(K)}$, combining them together to get an integrated node representation. Therefore, even if the DP noise overwhelms the signal in the higher-level aggregations, the information in the lower-level aggregations and/or the graph-agnostic features is still preserved and exploited for classification. As a result, regardless of the privacy budget, GAP is expected to always perform on par or better than pure MLP-based models that do not rely on the graph structure. We will empirically demonstrate this point in our experiments.

\subsection{Inference Mechanism}
GAP is compatible with both the transductive and the inductive inference, as discussed below. 

\paragraph{Transductive setting.}
In this setting, both training and inference are conducted on the same graph, but using different nodes for training and inference steps (\autoref{fig:transductive}). As the entire graph is available at training time, AM computes the private aggregations of all the nodes, including both training and test ones. Therefore, at inference time, we only give the cached aggregations of the test nodes to the trained CM to predict their labels. 

\paragraph{Inductive setting.}
Here, we use a new graph for inference different from the one used for training (\autoref{fig:inductive}). In this case, we first extract low-dimensional node features for the new graph using the pre-trained encoder and then feed them to AM to obtain the private aggregations. Finally, we input the private aggregations to the trained CM to get the node labels.

\end{revised}

\begin{algorithm}[t]
    \caption{GAP Training}
	\label{alg:gap}
	\footnotesize
	\SetKwInOut{Input}{Input}
	\SetKwInOut{Output}{Output}
	\ResetInOut{Output}
	\LinesNumbered
	\Input{Graph $\mathcal{G}=(\mathcal{V}, \mathcal{E})$ with adjacency matrix $\mathbf{A}$; node features $\mathbf{X}$; node labels $\mathbf{Y}$; max hop $K$; noise variance $\sigma^2$;}
	\Output{Trained model parameters $\{ \mathbf{\Theta^\star}_\text{enc}, \mathbf{\Theta^\star}_\text{base}^{(0)}, \dots, \mathbf{\Theta^\star}_\text{base}^{(K)}, \mathbf{\Theta^\star}_\text{comb}, \mathbf{\Theta^\star}_\text{head} \}$;
	}
	\BlankLine
	\DontPrintSemicolon
	Pre-train EM (\autoref{eq:encmod}) to obtain $\mathbf{\Theta^\star}_\text{enc}$.\;
	Use the pre-trained encoder (\autoref{eq:encoder}) to obtain encoded features $\mathbf{X}^{(0)}$.\;
	Row-normalize the encoded features (\autoref{eq:normalized}) to obtain $\check{\mathbf{X}}^{(0)}$.\;
	Use \autoref{alg:pma} to obtain private aggregations $\check{\mathbf{X}}^{(1)}, \dots, \check{\mathbf{X}}^{(K)}$.\;
	Train CM (\autoref{eq:mlpbase}-\ref{eq:mlphead}) to get $\mathbf{\Theta^\star}_\text{base}^{(0)}, \dots, \mathbf{\Theta^\star}_\text{base}^{(K)}, \mathbf{\Theta^\star}_\text{comb}$, $\mathbf{\Theta^\star}_\text{head}$.\;
	\Return  $\{\mathbf{\Theta^\star}_\text{enc}, \mathbf{\Theta^\star}_\text{base}^{(0)}, \dots, \mathbf{\Theta^\star}_\text{base}^{(K)}, \mathbf{\Theta^\star}_\text{comb}, \mathbf{\Theta^\star}_\text{head}\}$ \;
\end{algorithm}

\section{Privacy Analysis}\label{sec:priv}
\subsection{Edge-Level Privacy}\label{sec:edge-priv}

In the following, we provide a formal analysis of \name{}'s edge-level privacy guarantees at training and inference stages.

\paragraph{Training privacy.} The following arguments establish the DP guarantees of the PMA mechanism and the GAP training algorithm. The detailed proofs can be found in \autoref{sec:proof}.

\begin{theorem}\label{thm:privacy}
	Given the maximum hop $K\ge1$ and noise variance $\sigma^2$, the PMA mechanism presented in \autoref{alg:pma} satisfies edge-level $(\alpha, \nicefrac{K\alpha}{2\sigma^2})\text{-RDP}$ for any $\alpha > 1$.
\end{theorem}

\begin{proposition}\label{prop:gap-edge}
	For any $\delta\in(0,1)$, maximum hop $K\ge1$, and noise variance $\sigma^2$, \autoref{alg:gap} satisfies edge-level $(\epsilon,\delta)$-DP with $\epsilon=\frac{K}{2\sigma^2} + \nicefrac{\sqrt{2K\log{(1/\delta)}}}{\sigma}$.
\end{proposition}

\autoref{prop:gap-edge} shows that the privacy cost grows with the number of hops ($K$), but is independent of the number of training steps thanks to our GAP architecture.

\begin{revised}

	\paragraph{Inference privacy.}
	A major advantage of \name{} is that querying the model at inference time preserves DP \emph{without consuming additional privacy budget}. This is true for both the transductive and the inductive settings:
	\begin{itemize}[leftmargin=1em]
		\item \emph{Transductive setting:} In this setting, the inference is performed by feeding the privately trained CM with the cached aggregations of the test nodes, which have already been computed privately at training time. As this computation does not query the private graph structure and only post-processes the previous DP operations, due to the robustness of DP to post-processing, GAP provides inference privacy with no additional cost.
		
		\item \emph{Inductive setting:} In this case, first the new graph's node features are given to the encoder to obtain low-dimensional features, which are fed to AM to compute private aggregations. Then, the private aggregations are given to CM to obtain the final predictions. The only part where the private graph structure is queried is the AM, in which the PMA mechanism is applied to the new graph data, and thus the output is private. Furthermore, since the training and test graphs are disjoint, this application of the PMA mechanism is subject to the parallel composition of differentially private mechanisms, and thus it does not increase the privacy costs beyond that of training's. The other parts, the encoder and CM, perform graph-agnostic computations and only post-process previous DP outputs, leading to GAP ensuring inference privacy without extra privacy costs.
	\end{itemize}

\end{revised}

\subsection{Node-Level Privacy}\label{sec:node-priv}

Equipped with aggregation perturbation, the proposed GAP architecture guarantees edge-level privacy by default. However, it is readily extensible to provide node-level privacy guarantees as well, providing 
that we have bounded-degree graphs, i.e., the degree of each node should be bounded above by a constant $D$. This allows to bound the sensitivity of the aggregation function in the PMA mechanism when adding/removing a node, as in this case each node can influence at most $D$ other nodes. If the input graph has nodes with very high degrees, we can use neighbor sampling (as proposed in~\cite{daigavane2021node}) to randomly sample at most $D$ neighbors per node.

For bounded-degree graphs, adding or removing a node corresponds (in the worst case) to adding or removing $D$ edges. Therefore, our PMA mechanism also ensures node-level privacy, albeit with increased privacy costs compared to the edge-level setting (see \autoref{thm:nodeprivacy} below).

However, since the node features and labels are also private under node-level DP, both EM and CM need to be trained privately as they access node features/labels. To this end, we can simply use standard DP-SGD~\cite{abadi2016deep} or any other differentially private learning algorithm for pre-training the encoder as well as training CM with DP. In other words, steps 1 and 5 of \autoref{alg:gap} must be done with DP instead of regular non-private training. This way, since each of the three \name{} modules become node-level private, the entire \name{} model, as an adaptive composition of several node-level private mechanisms, satisfies node-level DP. 
The formal node-level privacy analysis of GAP's training and inference is provided below.

\paragraph{Training privacy.} The node-level privacy guarantees of the PMA mechanism and the GAP training algorithm are as follows. Detailed proofs are deferred to \autoref{sec:proof}.

\begin{theorem}\label{thm:nodeprivacy}
	Given the maximum degree $D\ge1$, maximum hop $K\ge1$, and noise variance $\sigma^2$, \autoref{alg:pma} (PMA mechanism) satisfies node-level $(\alpha, \nicefrac{DK\alpha}{2\sigma^2})\text{-RDP}$ for any $\alpha > 1$.
\end{theorem}

\begin{proposition}\label{prop:gap-node}
	For any $\alpha > 1$, let encoder pre-training (Step 1 of \autoref{alg:gap}) and CM training (Step 5 of \autoref{alg:gap}) satisfy $(\alpha,\epsilon_1(\alpha))$-RDP and $(\alpha,\epsilon_5(\alpha))$-RDP, respectively. Then, for any $0 < \delta < 1$, maximum hop $K\ge1$, maximum degree $D\ge1$, and noise variance $\sigma^2$, \autoref{alg:gap} satisfies node-level $(\epsilon, \delta)$-DP with $\epsilon = \epsilon_1(\alpha) + \epsilon_5(\alpha) + \nicefrac{DK\alpha}{2\sigma^2} + \nicefrac{\log(1/\delta)}{\alpha-1}$.
\end{proposition}

Note that in \autoref{prop:gap-node}, we cannot optimize $\alpha$ in closed form as we do not know the precise form of $\epsilon_1(\alpha)$ and $\epsilon_5(\alpha)$. However, in our experiments, we numerically optimize the choice of $\alpha$ on a per-case basis.

\paragraph{Inference privacy.}
The arguments stated for edge-level inference privacy also hold for node-level privacy. Note that in the inductive setting, the test graph should also have bounded degree
for the node-level inference privacy guarantees to hold.

\section{Discussion}
\label{sec:discussion}

\paragraph{Choice of aggregation function.}
In this paper, we used \textsc{Sum} as the default choice of aggregation function. Although other choices of aggregation functions are also possible, we empirically found that \textsc{Sum} is the most efficient choice to privatize, as its sensitivity does not depend on the size of the aggregation set (i.e., number of neighbors), which is itself a quantity that should be computed privately. For example, the calculation of both \textsc{Mean} and GCN~\cite{kipf2017semi} aggregation functions depend on the node degrees, and thus requires additional privacy budget to be spent on perturbing node degrees. 
In any case, \textsc{Sum} is recognized as one of the most expressive aggregation functions in the GNN literature~\cite{xu2018how, corso2020pna}.

\paragraph{Normalization instead of clipping.}
The PMA mechanism uses normalization to bound the effect of each individual feature on the \textsc{Sum} aggregation function. While clipping is more common in the private learning literature (e.g., gradient clipping in DP-SGD \cite{abadi2016deep}), we empirically found that normalization is a better choice for aggregation perturbation: CM is then trained on normalized data, which tends to facilitate learning. Normalizing the node embeddings is actually commonly done in non-private GNNs as well to stabilize training~\cite{hamilton2017inductive, you2020design}.

\paragraph{Limitations.}
As the PMA mechanism adds random noise to the aggregation function, its utility naturally depends on the size of the node's aggregation set, i.e., the node's degree. 
Specifically, with a certain amount of noise, the more inbound neighbors a node has, the more accurate its noisy aggregated vector will be. 
This implies that graphs with higher average degree per node can tolerate larger noise in the aggregation function, and thus GAP can achieve a better privacy-accuracy trade-off on such graphs. 
Conversely, GAP's performance will suffer if the average degree of the graph is too low, requiring higher privacy budgets to achieve acceptable accuracy. 
Note however that this is an expected behavior: nodes with fewer inbound neighbors are more easily influenced by a change in their neighborhood compared to nodes with higher degrees, and thus the privacy of low-degree nodes is harder to preserve than high-degree ones. Furthermore, 
this limitation is not specific to GAP: it is shared by all DP algorithms, whose performance generally suffer from lack of sufficient data.

\begin{revised}
    \paragraph{Edge-level vs. node-level privacy.} While GAP can work in either edge-level or node-level privacy settings, it must be emphasized that the former setting is suitable only for the use cases where the node-level information (e.g, features or labels) is not sensitive or is publicly available (e.g., the vertically partitioned graph setting described in \cite{wu2021linkteller}). Whenever node-level information is private as well (e.g., user profiles in a social network), however, edge-level privacy fails to provide appropriate privacy protection, and thus node-level privacy setting has to be enforced.
\end{revised}
\section{Experiments}\label{sec:exp}

\begin{revised}
    In this section, we conduct extensive experiments to empirically evaluate GAP's privacy-accuracy performance and its resilience under privacy attacks. As GAP's privacy guarantees are the same under both transductive and inductive settings, we only focus on the former, which has also more pertinent use cases (e.g., social networks).
\end{revised}

\subsection{Datasets}
We evaluate the proposed method on three publicly available node classification datasets, which are medium to large scale in terms of the number of nodes and edges: 

\paragraph{Facebook~\cite{traud2012social}.} This dataset contains the anonymized Facebook social network between UIUC students collected in September 2005. Nodes represent Facebook users and edges indicate friendship. Each node (user) has the following attributes: 
student/faculty status, gender, major, minor, and housing status, and the task is to predict the class year of users.

\paragraph{Reddit~\cite{hamilton2017inductive}.} This dataset consist of a set of posts from the Reddit social network, where each node represents a post and an edge indicates if the same user commented on both posts. Node features are extracted based on the embedding of the post contents, and the task is to predict the community (subreddit) that a post belongs to.

\paragraph{Amazon~\cite{chiang2019cluster}.} The largest dataset used in this paper represents Amazon product co-purchasing network, where nodes represent products sold on Amazon and an edge indicates if two products are purchased together. Node features are bag-of-words vectors of the product description followed by PCA, and the task is to predict the category of the products.\\

We preprocess the datasets by limiting the classes to those having 1k, 10k, and 100k nodes on Facebook, Reddit, and Amazon, respectively.
We then randomly split the remaining nodes into training, validation, and test sets with 75/10/15\% ratios, respectively. \autoref{tab:dataset} summarizes the statistics of the datasets after preprocessing.

\begin{table}[t]
  \centering
  \footnotesize
  \caption{Overview of dataset statistics.}\label{tab:dataset}
  \sc
  \begin{tabu} to \columnwidth {X[l,1.5] X[c,1.5] X[c,1.5] X[c,1] X[c,1] X[c,1]}
    \toprule
    {Dataset} & {Nodes} & {Edges} & {Degree} & {Features} & {Classes} \\
    \midrule
    Facebook & {26,406}  & {2,117,924} & 62 & 501 & 6    \\
    Reddit & {116,713}  & {46,233,380} & 209 & 602  & 8 \\
    Amazon & {1,790,731}  & {80,966,832} & 22 & 100  & 10  \\
    \bottomrule
  \end{tabu}
\end{table}

\subsection{Competing Methods} 


\paragraph{Edge-level private methods.} The following methods are evaluated under edge-level privacy: 
\begin{itemize}[leftmargin=1em, noitemsep, topsep=1pt]
    \item \emph{GAP-EDP:} Our proposed edge-level DP algorithm.
    \item \emph{SAGE-EDP:} This is the method of Wu~\etal~\cite{wu2021linkteller} that uses the graph perturbation approach, with the popular GraphSAGE architecture~\cite{hamilton2017inductive} as its backbone GNN model. 
    \begin{revised}
        We perturb the graph's adjacency matrix using the Asymmetric Randomized Response (ARR)~\cite{imola2021comm}, which performs better than \textsc{EdgeRand}~\cite{wu2021linkteller} by limiting the output sparsity.
    \end{revised}
    \item \emph{MLP:} A simple MLP model that does not use the graph edges, and thus provides perfect edge-level privacy ($\epsilon=0$).
\end{itemize}

\paragraph{Node-level private methods.} We compare the following node-level private algorithms:
\begin{itemize}[leftmargin=1em, noitemsep, topsep=1pt]
    \item \emph{GAP-NDP:} Our proposed node-level DP approach.
    \item \emph{SAGE-NDP:} This is the method of Daigavane~\etal~\cite{daigavane2021node} that adapts the standard DP-SGD method for 1-layer GNNs, with the same GraphSAGE architecture as its backbone model. Since this method does not inherently ensure inference privacy, as suggested by its authors, we add noise to the aggregation function based on its node-level sensitivity at test time and account for the additional privacy cost.
    \item \emph{MLP-DP:} Similar to MLP, but trained with DP-SGD so as to provide node-level DP without using the graph edges.
\end{itemize}

We do not consider the approach of \cite{olatunji2021releasing} as it requires public graph data and is thus not directly comparable to the others.

\paragraph{Non-private methods.} To quantify the accuracy loss of private approaches, we use the following non-private methods ($\epsilon=\infty$): 
\begin{itemize}[leftmargin=1em, noitemsep, topsep=1pt]
\item \emph{GAP-$\infty$:} a non-private counterpart of the GAP method, where we do not perturb the aggregations. 
\item \emph{SAGE-$\infty$:} a non-private GraphSAGE model.
\end{itemize}

\subsection{Experimental Setup}\label{sec:setup}

\paragraph{Model implementation details.} 
For our GAP models (GAP-EDP, GAP-NDP, and GAP-$\infty$), we set the number of $\text{MLP}_\text{enc}$, $\text{MLP}_\text{base}$, and $\text{MLP}_\text{head}$ layers to be 2, 1, and 1, respectively. We use concatenation as the \textsc{Combine} function (\autoref{eq:comb}) and tune the number of hops $K$ in $\{1,2,\dots,5\}$. 
For the GraphSAGE models (SAGE-EDP, SAGE-NDP, and SAGE-$\infty$), we use the \textsc{Sum} aggregation function and tune the number of message-passing layers in $\{1,2,\dots,5\}$, except for SAGE-NDP that only supports one message-passing layer. We use a 2-layer and a 1-layer MLP as preprocessing and post-processing before and after the message-passing layers, respectively.
For the MLP baselines (MLP and MLP-DP), we set the number of layers to 3. In addition, for both the GAP-NDP and SAGE-NDP methods, \begin{revised}we use randomized neighbor sampling to bound the maximum degree $D$ and search for the best $D$ within  $\{100,200,300,400\}$.\end{revised}
For all methods, we set the number of hidden units to 16 (including the dimension of GAP's encoded representation) and use the SeLU activation function~\cite{klambauer2017self} at every layer. Batch-normalization is used for all methods except the node-level private ones (GAP-NDP, SAGE-NDP, and MLP-DP), for which batch-normalization is not supported. 

\paragraph{Training and evaluation details.} We train the non-private and edge-level private methods using the Adam optimizer over 100 epochs with full-sized batches. 
For the node-level private algorithms (GAP-NDP, SAGE-NDP, MLP-DP), we use DP-Adam~\cite{gylberth2017differentially} with maximum gradient norm set to 1, and train each model for 10 epochs with a batch size of 256, 2048, 4096 on Facebook, Reddit, and Amazon, respectively. 
For our GAP models (GAP-$\infty$, GAP-EDP, and GAP-NDP), we use the same parameter setting for training both the encoder and classification modules. 
We train all the methods with a learning rate of 0.01 and repeat each combination of possible hyperparameter values 10 times. We pick the best performing model based on validation accuracy, and report the average test accuracy with 95\% confidence interval calculated by bootstrapping with 1000 samples.

\paragraph{Privacy accounting and calibration.} Privacy budget accounting is done via the Analytical Moments Accountant~\cite{pmlr-v89-wang19b}. We numerically calibrate the noise scale (i.e., the noise standard deviation $\sigma$ divided by the sensitivity) of PMA (for GAP-EDP and GAP-NDP), ARR (for SAGE-EDP), DP-SGD (for GAP-NDP, SAGE-NDP, and MLP-DP) and the Gaussian mechanism (for inference privacy in SAGE-NDP) to achieve the desired $(\epsilon,\delta)$-DP. We report results for several values of $\epsilon$, while $\delta$ is set to be smaller than the inverse number of private entities (i.e., edges for edge-level privacy, nodes for node-level privacy).
For both GAP-NDP and SAGE-NDP, we use the same noise scale for perturbing the gradients (in DP-SGD) and the aggregations (in PMA and Gaussian mechanisms).

\paragraph{Software and hardware.} All the models are implemented in PyTorch~\cite{NEURIPS2019_9015} using PyTorch-Geometric (PyG)~\cite{Fey/Lenssen/2019}. We use the \texttt{autodp} library\footnote{\url{https://github.com/yuxiangw/autodp}} which implements analytical moments accountant, and utilize Opacus~\cite{opacus} for training the node-level private models with differential privacy. Experiments are conducted on Sun Grid Engine with NVIDIA GeForce RTX 3090 and NVIDIA Tesla V100 GPUs, Intel Xeon 6238 CPUs, and 32 GB RAM. 

\begin{table}[t]
    \centering
    \caption{Test accuracy of different methods on the three datasets. The best performing method in each category --- none-private, edge-level DP and node-level DP --- is highlighted.}
    \label{tab:results}
    \sc
    \footnotesize
    \begin{tabular}{llcccc}
        \toprule
         & Method & $\epsilon$ & Facebook &           Reddit &           Amazon \\
        \midrule
        \multirow{2}{*}{\rotatebox[origin=c]{90}{\scriptsize None}} 
         & GAP-$\infty$ & $\infty$ &  80.0 $\pm$ 0.48 &  \textbf{\color{BrickRed} 99.4 $\pm$ 0.02} &  91.2 $\pm$ 0.07 \\
         & SAGE-$\infty$ & $\infty$ &  \textbf{\color{BrickRed} 83.2 $\pm$ 0.68} &  99.1 $\pm$ 0.01 &  \textbf{\color{BrickRed} 92.7 $\pm$ 0.09} \\
        \midrule
        \multirow{3}{*}{\rotatebox[origin=c]{90}{\scriptsize Edge DP}} 
         & GAP-EDP & 4 &  \textbf{\color{ForestGreen} 76.3 $\pm$ 0.21} &  \textbf{\color{ForestGreen} 98.7 $\pm$ 0.03} &  \textbf{\color{ForestGreen} 83.8 $\pm$ 0.26} \\
         & SAGE-EDP & 4 &  50.4 $\pm$ 0.69 &  84.6 $\pm$ 1.63 &  68.3 $\pm$ 0.99 \\
         & MLP & 0 &  50.8 $\pm$ 0.17 &  82.4 $\pm$ 0.10 &  71.1 $\pm$ 0.18 \\
        \midrule
        \multirow{3}{*}{\rotatebox[origin=c]{90}{\scriptsize Node DP}} 
         & GAP-NDP & 8 &  \textbf{\color{RoyalBlue} 63.2 $\pm$ 0.35} &  \textbf{\color{RoyalBlue} 94.0 $\pm$ 0.14} &  \textbf{\color{RoyalBlue} 77.4 $\pm$ 0.07} \\
         & SAGE-NDP & 8 &  37.2 $\pm$ 0.96 &  60.5 $\pm$ 1.10 &  27.5 $\pm$ 0.83 \\
         & MLP-DP & 8 &  50.2 $\pm$ 0.25 &  81.5 $\pm$ 0.12 &  73.6 $\pm$ 0.05 \\
        \bottomrule
    \end{tabular}
    
\end{table}

\begin{figure*}[t]
    \centering
    \includegraphics[width=\columnwidth]{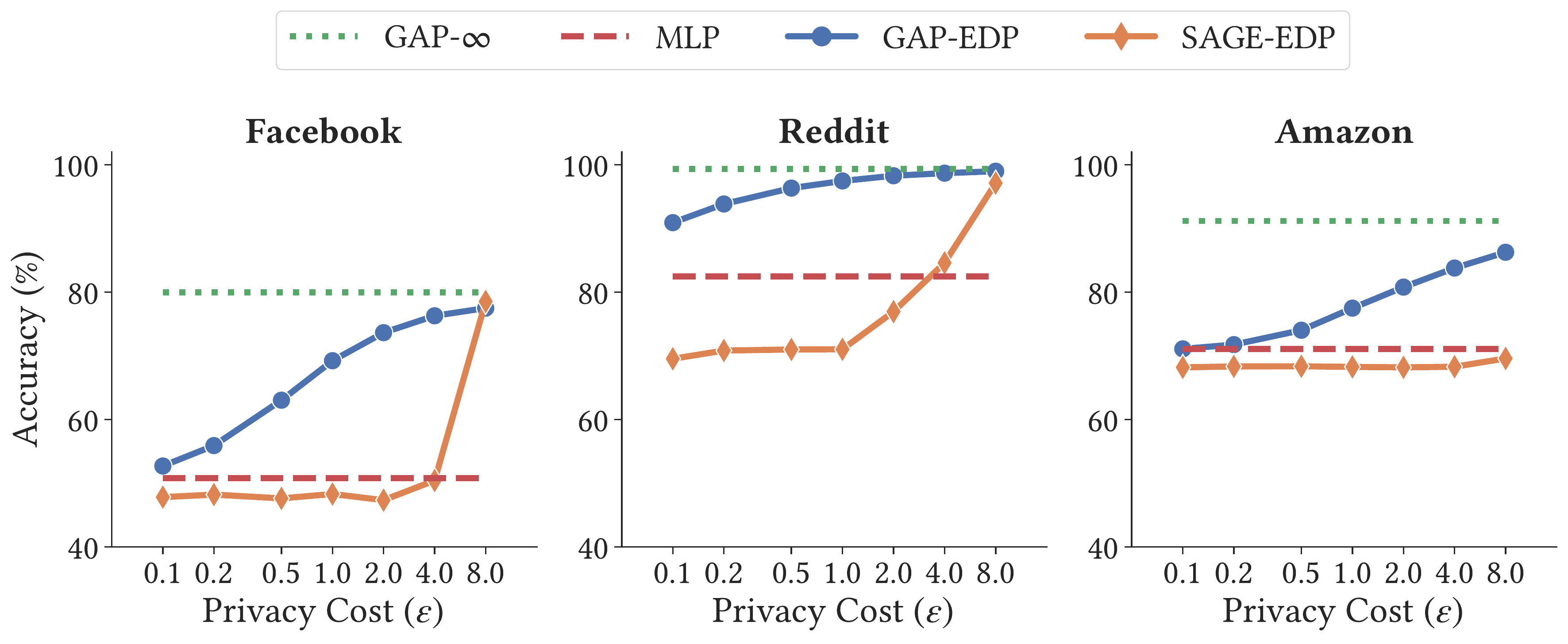}\hfill
    \includegraphics[width=\columnwidth]{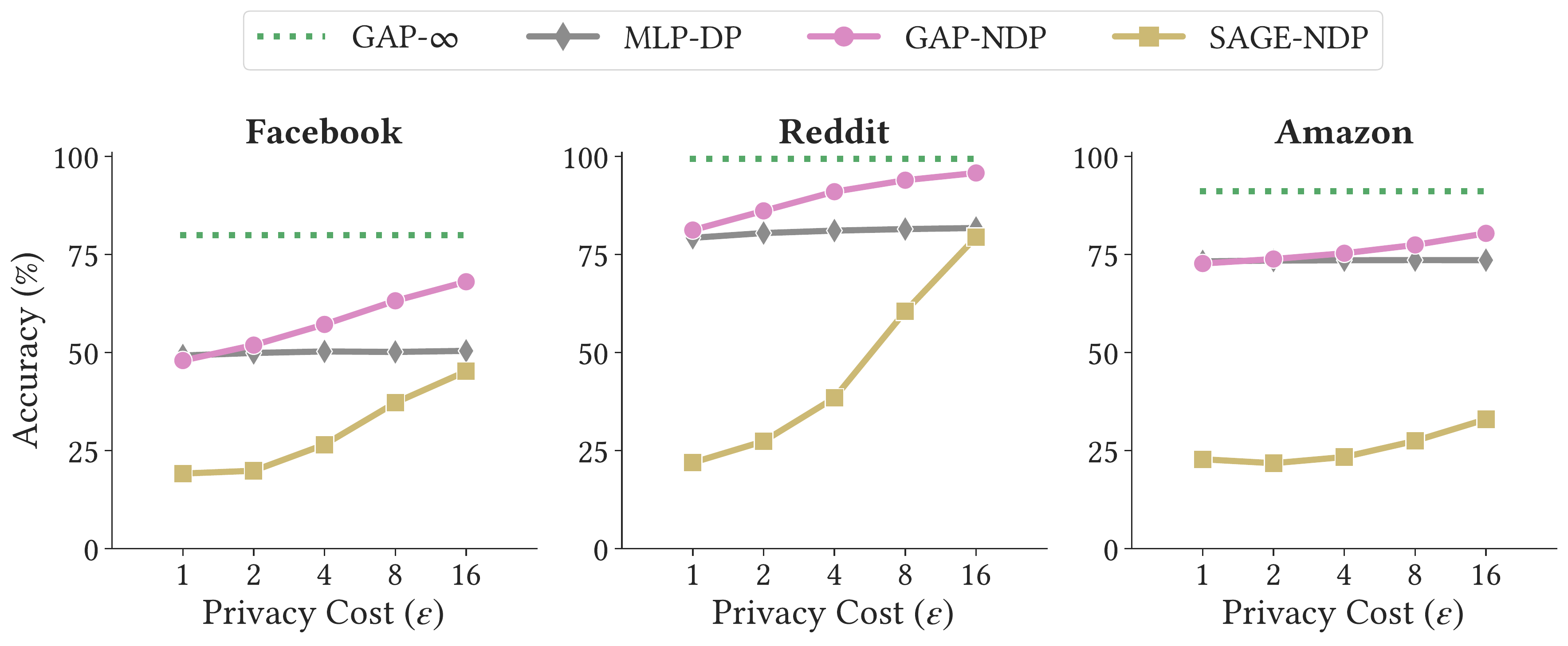}
    \caption{Accuracy vs. privacy cost ($\epsilon$) of edge-level private algorithms (left) and node-level private methods (right). 
    }
    \label{fig:acc-priv}
\end{figure*}

\subsection{Experimental Results}

\subsubsection{Trade-offs between Privacy and Accuracy}

We first compare the accuracy of our proposed methods against the non-private, edge-level private, and node-level private baselines.
\begin{revised}
    We fix the privacy budget to $\epsilon=8$ for the node-level private methods and $\epsilon=4$ for the edge-level private ones    
\end{revised}
(except for MLP, which does not use the graph structure and thus achieves $\epsilon=0$). The results are presented in \autoref{tab:results}.
We observe that in the non-private setting, the proposed GAP architecture is competitive with SAGE, with only a slight decrease in accuracy on Facebook and Amazon. Under both edge-level and node-level privacy settings, however, our proposed methods GAP-EDP and GAP-NDP significantly outperform their competitors.
\begin{revised}
    Particularly, under edge-level privacy, GAP-EDP's accuracy is roughly 26, 14, and 15 points higher than the best competitor over Facebook, Reddit, and Amazon, respectively. Under node-level privacy, our proposed GAP-NDP method outperforms the best performing competitor by approximately 13, 13, and 4 accuracy points, respectively.
\end{revised}

Next, to investigate how different methods perform under different privacy budgets, 
\begin{revised}
    we vary $\epsilon$ from 0.1 to 8 for edge-level private methods and from 1 to 16 for node-level private algorithms    
\end{revised}
and report the accuracy of the methods under each privacy budget. The result for both edge-level and node-level privacy settings is depicted in \autoref{fig:acc-priv}.

Under edge-level privacy (\autoref{fig:acc-priv}, left side), we observe that GAP-EDP consistently outperforms its direct competitor, SAGE-EDP, especially at lower privacy costs.
\begin{revised}
    The relative gap between GAP-EDP and SAGE-EDP is influenced by the average degree of the dataset. For example, on Facebook and Reddit with higher average degrees, SAGE-EDP requires a high privacy budget of $\epsilon\geq 8$ to achieve reasonable accuracy, but on Amazon, which has the lowest average degree, it cannot even beat the MLP baseline. In comparison, the accuracy of GAP-EDP approaches the non-private GAP-$\infty$ at much lower privacy budgets, and always performs better than a vanilla MLP. This is because SAGE-EDP perturbs the adjacency matrix, which is extremely high-dimensional and sparse, while GAP-EDP perturbs the aggregated node embeddings, which has much lower dimensions and is not sparse compared to the adjacency matrix.
\end{revised}
The amount of accuracy loss with respect to the non-private method also depends on the average degree of the graph. For example, on Reddit at $\epsilon=2$, GAP-$\infty$'s accuracy is only 1 point higher than GAP-EDP's, while on Amazon at $\epsilon=8$, GAP-EDP's accuracy fall behind GAP-$\infty$ by around 5 points. These observations are in line with our discussion of~\autoref{sec:discussion}.

\looseness=-1 We can observe similar trends under node-level privacy (\autoref{fig:acc-priv}, right side). 
\begin{revised}
    We see that our GAP-NDP method always performs on par or better than the MLP-DP baseline, and also significantly outperforms SAGE-NDP under all the considered privacy budgets. We attribute this to two factors: first, SAGE-NDP is limited to 1-layer models and thus cannot exploit higher-order aggregations; second, the naive noisy aggregation patch for supporting inference privacy severely hurts the performance of SAGE-NDP.
\end{revised}
As expected, since the node-level private GAP-NDP hides more information (e.g., node features, labels, and all the adjacent edges to a node) than the edge-level private GAP-EDP, it requires larger privacy budgets to achieve a reasonable accuracy. Still, the accuracy loss with respect to the non-private method is higher in the node-level private method as we have further information loss due to neighborhood sampling (to bound the graph's maximum degree) and gradient clipping (to bound the sensitivity in DP-SGD/Adam).

\begin{revised}
\subsubsection{Resilience Against Privacy Attacks}
As mentioned above, the node-level private methods require a higher privacy budget than the edge-level private ones as they attempt to hide much more information. 
In order to assess the practical implications of choosing rather large privacy budgets (e.g., $\epsilon=8$ in~\autoref{tab:results}), we empirically measure the privacy guarantees of GAP-NDP and other node-level private methods by conducting node-level membership inference attack~\cite{he2021node, olatunji2021membership} as the most relevant adapted privacy attack to GNNs.

\paragraph{Attack overview.}
The attack is modeled as a binary classification task, where the goal is to infer whether an arbitrary node $v$ is a member of the training set $\mathcal{V}_T$ of the target GNN. 
The key intuition is that due to overfitting, GNNs give more confident probability scores to training nodes than to test ones, which can be exploited by the attacker to distinguish members of the training set. Having access to a shadow graph dataset coming from the same distribution as the target graph, the attacker first trains a shadow GNN to mimic the behavior of the target GNN, but for which the membership ground truth is known. Then, the attacker trains an attack model over the probability scores of the shadow graph nodes and their corresponding membership labels. Finally, the attacker uses the trained attack model to infer the membership of the target graph nodes.

\paragraph{Attack settings.}
We follow the TSTF (train on subgraph, test on full graph) approach of~\cite{olatunji2021membership} for the node-level membership inference attack. Specifically, we consider a strong adversary with access to a shadow graph dataset with 1000 nodes per class, which are sampled uniformly at random from the target dataset. For the shadow model, we use the same architecture and hyperparameters as the target model (described in~\autoref{sec:setup}). Similar to prior work~\cite{olatunji2021membership}, we use a 3-layer MLP with 64 hidden units as the attack model, and use the area under the receiver operating characteristic curve (AUC) averaged over 10 runs as the evaluation metric.

\paragraph{Results.}
\autoref{tab:attack} reports the mean AUC of the attack on different node-level private methods trained with the same setting as in~\autoref{fig:acc-priv} (right). As we see, the attack is quite effective on the non-private methods ($\epsilon=\infty$), especially on Facebook and Amazon datasets. The success of the attack on each method mainly depends on its generalization gap (the difference between the training and test accuracy): the higher the generalization gap, the more confident the model is on the training nodes and the easier it is to distinguish them from the test nodes. Hence, the lower attack performance on the non-private SAGE method is due to its lower generalization gap compared to the other methods.
Nevertheless, for all private GNN methods, we observe that DP with privacy budgets as large as $\epsilon=16$ can effectively defend against the attack, reducing the AUC to about 50\% (random baseline) on all datasets. This result is in line with the work of~\cite{jayaraman2019evaluating, jagielski2020auditing, nasr2021adversary}, showing that DP with large privacy budgets can still effectively mitigate realistic membership inference attacks.

\begingroup
\setlength{\tabcolsep}{3pt} 
\begin{table}
    \centering
    \caption{Mean AUC of node membership inference attack.}
    \label{tab:attack}
    \footnotesize
    \sc
    \begin{tabular}{llcccccc}
        \toprule
        Dataset & Method & $\epsilon=1$ & $\epsilon=2$ & $\epsilon=4$ & $\epsilon=8$ & $\epsilon=16$ & $\epsilon=\infty$ \\
        \midrule
                    & GAP-NDP &        50.16 &         50.25 &        50.61 &       51.11 &         52.66 &         81.67 \\
        Facebook    & SAGE-NDP &       50.25 &         50.20 &        50.23 &                    50.17 &         50.20 &         62.49 \\
                     & MLP-DP &         50.32 &         50.72 &        52.13 &       53.44 &         54.77 &         81.57 \\
        \midrule
                    & GAP-NDP &        50.04 &        50.39 &        51.20 &          52.23 &         52.54 &             54.97 \\
        Reddit  & SAGE-NDP &        49.97 &        49.97 &        49.95 &                     50.00 &         49.98 &             50.05 \\
                    & MLP-DP &        51.25 &        53.09 &        55.13 &           56.72 &         58.32 &             71.35 \\
        \midrule
                    & GAP-NDP &        50.06 &        50.23 &        50.54 &          51.53 &         51.72 &             66.68 \\
        Amazon  & SAGE-NDP &        49.93 &        49.93 &        49.93 &                     49.92 &         49.97 &             59.41 \\
                    & MLP-DP &        50.30 &        50.58 &        51.43 &           52.31 &         53.34 &             72.97 \\
        \bottomrule
    \end{tabular}
\end{table}
\endgroup
\end{revised}

\subsubsection{Ablation Studies}

\paragraph{Effectiveness of the encoder module (EM).}
In this experiment, we investigate the effect of EM on the accuracy/privacy performance of the proposed methods, GAP-EDP and GAP-NDP. We compare the case in which EM is used as usual with the case where we remove EM and just input the original node features to the aggregation module. The results under different privacy budgets are given in \autoref{fig:encoder}. We can observe that in all cases, the accuracy of GAP-EDP and GAP-NDP is higher with EM than without it. For example, leveraging EM results in a gain of around 20, 2, and 5 accuracy points for GAP-EDP with $\epsilon=1$ on Facebook, Reddit, and Amazon datasets, respectively. GAP-NDP with EM also benefits from a gain of more than 10, 10, and 5 points with $\epsilon=4$ on Facebook, Reedit, and Amazon datasets, respectively.
\begin{revised}
    As discussed in~\autoref{sec:em}, the improved performance with EM is mainly due to the reduced dimensionality of the aggregation module's input, which leads to adding less noise to the aggregations.
\end{revised}
Also, the effect of EM is more significant on GAP-NDP, as the amount of noise injected into the aggregations is generally larger for node-level privacy, hence dimensionality reduction becomes more critical to mitigate the impact of noise.

\begin{figure}[t]
    \centering
    \includegraphics[width=\columnwidth]{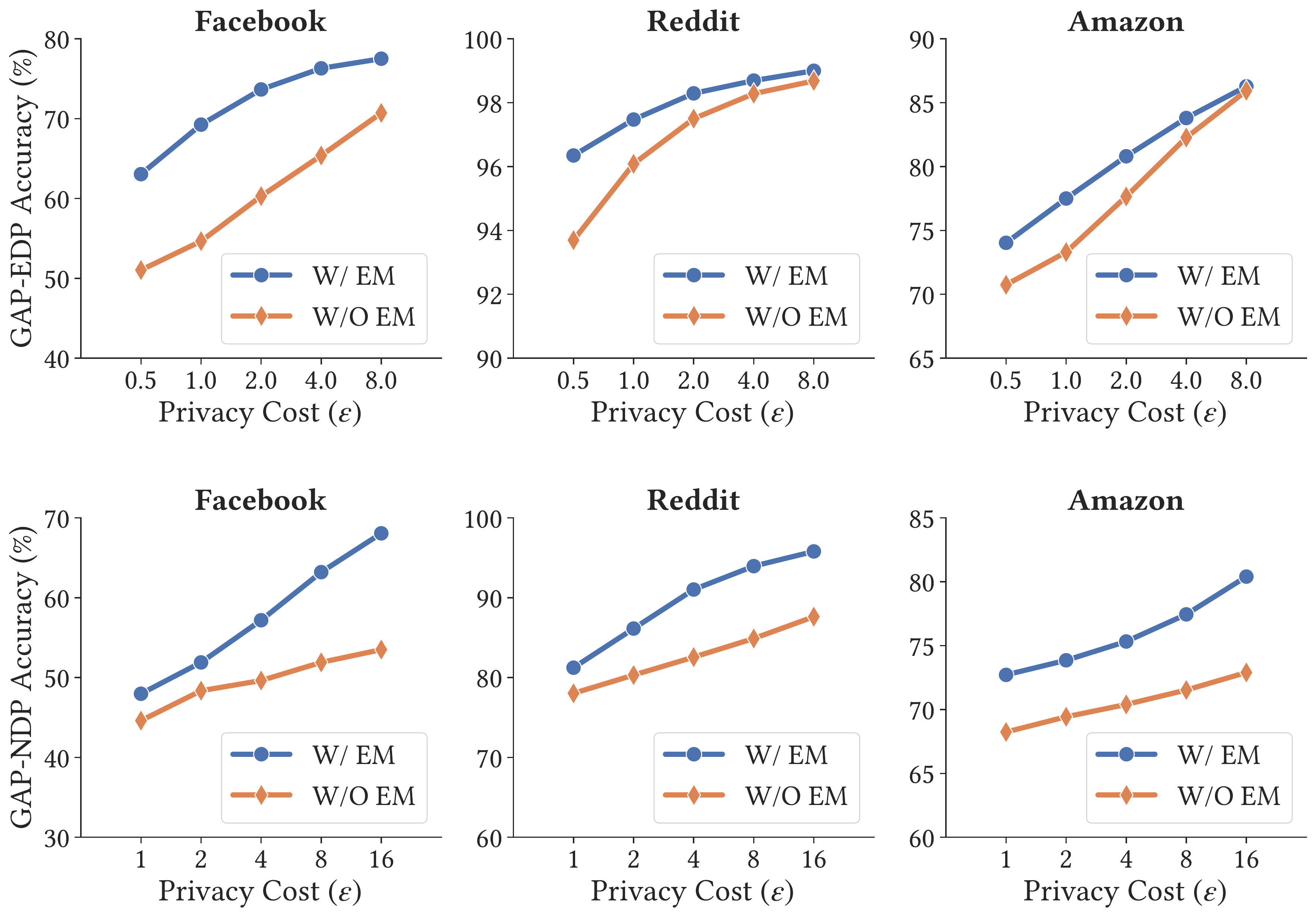}
    \caption{Effect of the encoder module (EM) on the accuracy/privacy performance of the edge-level private GAP-EDP (top) and the node-level private GAP-NDP (bottom).}
    \label{fig:encoder}
\end{figure}

\begin{figure}[t]
    \centering
    \includegraphics[width=\columnwidth]{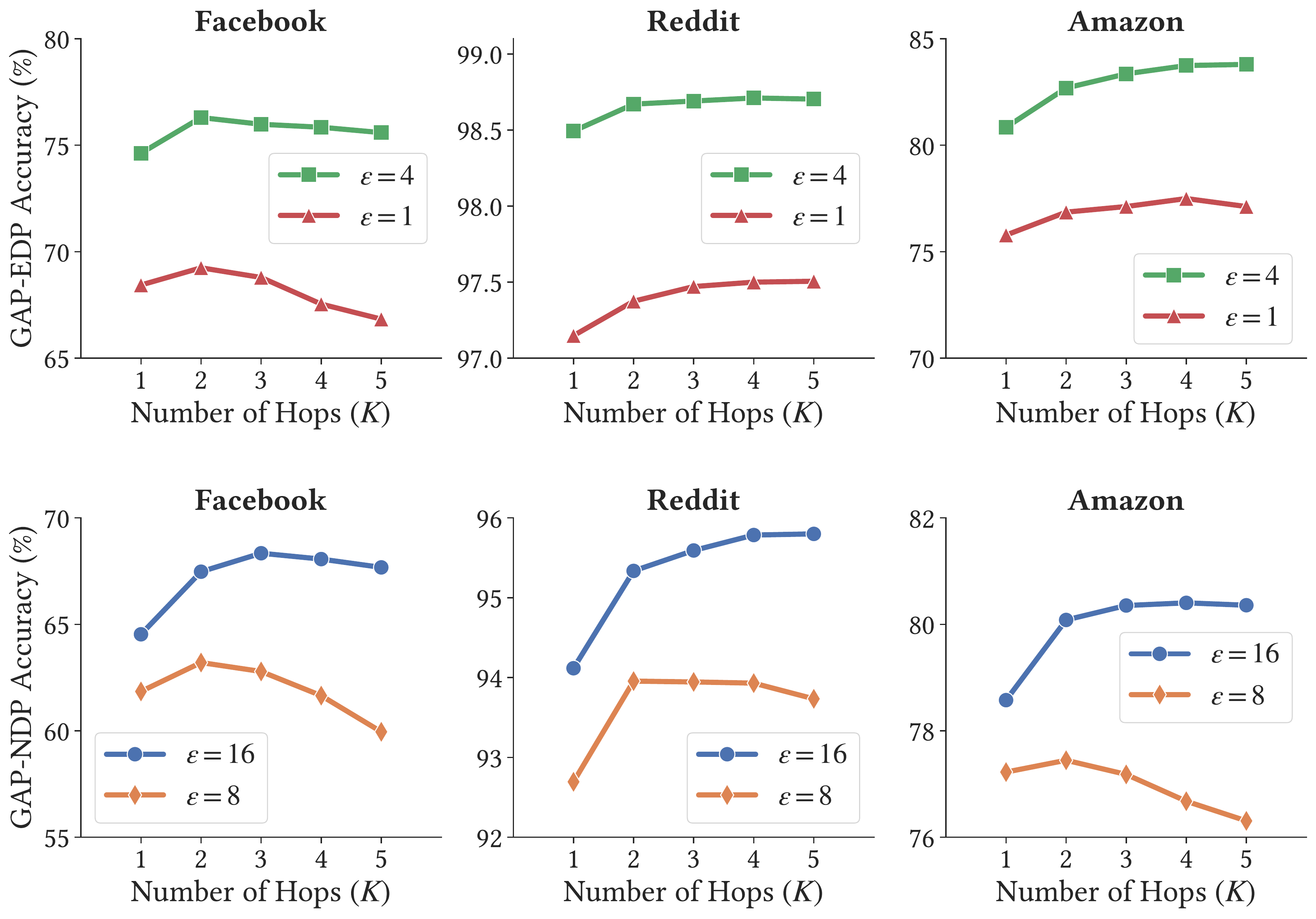}
    \caption{Effect of the number of hops $K$ on the accuracy/privacy performance of the edge-level private GAP-EDP (top) and the node-level private GAP-NDP (bottom).
    }
    \label{fig:hops}
\end{figure}

\paragraph{Effect of the number of hops.}
In this experiment, we investigate how changing the number of hops $K$ affects the accuracy/privacy performance of our proposed methods, GAP-EDP and GAP-NDP. We vary $K$ within $\{1,2,3,4,5\}$ and report the accuracy under different privacy budgets: $\epsilon\in\{1,4\}$ for GAP-EDP and $\epsilon\in\{8,16\}$ for GAP-NDP. The result is depicted in \autoref{fig:hops}. We observe that 
both of our methods can effectively benefit from allowing multiple hops, but there is a trade-off in increasing the number of hops. As we increment $K$, the accuracy of both GAP-EDP and GAP-NDP method increase up to a point and then steady or decrease in almost all cases. The reason is that with a larger $K$ the model is able to utilize information from more distant nodes (all the nodes within the $K$-hop neighborhood of a node) for prediction, which can increase the final accuracy. However, as more hops are involved, the amount of noise in the aggregations is also increased, which adversely affects the model's accuracy. We can see that with the lower privacy budgets where the noise is more severe, both GAP-EDP and GAP-NDP achieve their peak accuracy at smaller $K$ values. But as the privacy budget increases, the magnitude of the noise is reduced, enabling the models to benefit from larger $K$ values.

\begin{figure}[t]
    \centering
    \includegraphics[width=\columnwidth]{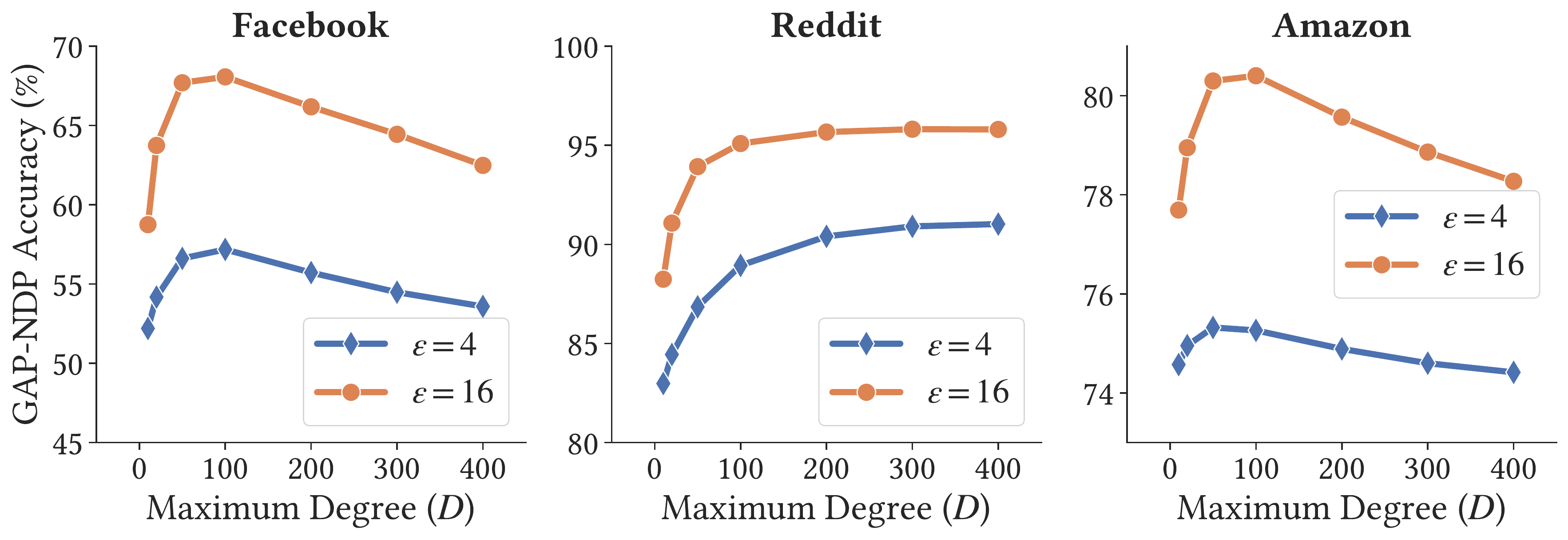}
    \caption{Effect of the degree bound $D$ on the accuracy/privacy performance of the node-level private GAP-NDP method. 
    }
    \label{fig:degree}
\end{figure}

\paragraph{Effect of the maximum degree.}
We now analyze the effect of $D$ on the performance of our node-level private method. 
\begin{revised}
    We vary $D$ from 10 to 400 and report GAP-NDP's accuracy under two different privacy budgets $\epsilon\in\{4,16\}$. \autoref{fig:degree} shows that the accuracy keeps growing with $D$ on Reddit (which has a high average degree), while on Facebook and Amazon (lower average degrees) the accuracy increases with $D$ up to a peak point, and drops afterwards. 
\end{revised}
This is due to the trade-off between having more samples for aggregation and the amount of noise injected: the larger $D$, the fewer neighbors are excluded from the aggregations (i.e., less information loss), but on the other hand, the larger the sensitivity of the aggregation function, leading to more noise injection. We also observe that the accuracy gain as a result of increasing $D$ gets bigger as the privacy budget is increased from 5 to 20, since a higher privacy budget compensates for the higher sensitivity by reducing the amount of noise.

\section{Conclusion}\label{sec:conclusion}

In this paper, we presented \name{}, a privacy-preserving GNN architecture that ensures both edge-level and node-level differential privacy for training and inference over sensitive graph data.
We used aggregation perturbation, where the Gaussian mechanism is applied to the output of the GNN's aggregation function, as a fundamental technique to achieve DP in our approach. 
\begin{revised}
    We proposed a new GNN architecture tailored to the specifics of private learning over graphs, aiming to achieve better privacy-accuracy trade-offs while tackling the intricate challenges involved in the design of differentially private GNNs.
\end{revised}
Experimental results over real-world graph datasets showed that our approach achieves favorable privacy/accuracy trade-offs and significantly outperforms existing methods.
Promising future directions include: (i) investigating robust aggregation functions that provide specific benefits for private learning; (ii) exploiting the redundancy of information in recursive aggregations to achieve tighter composition when the number of hops $K$ gets large, which might prove useful for specific applications; (iii) extending the framework to other tasks and scenarios, such as link-wise prediction or learning over dynamic graphs; and (iv) conducting an extended theoretical analysis of differentially private GNNs, such as proving utility bounds and characterizing their expressiveness.

\section*{Acknowledgments}
This work was supported by the European Commission's Horizon 2020 Program ICT-48-2020, under grant number 951911, AI4Media project.
It was also supported by the French National Research Agency (ANR) through grant ANR-20-CE23-0015 (Project PRIDE). Ali Shahin Shamsabadi acknowledges support from The Alan Turing Institute.

\section*{Availability}
Our open-source implementation is publicly available on GitHub at \href{https://github.com/sisaman/GAP}{https://github.com/sisaman/GAP}.

\bibliographystyle{plain}
\bibliography{ref}

\appendix
\balance
\section{Deferred Theoretical Arguments}\label{sec:proof}

\subsection{Proof of \autoref{thm:privacy}}

To prove \autoref{thm:privacy}, we first establish the following lemma.

\begin{lemma}\label{lem:e-sensitivity}
	Let $\AGG{(\mathbf{X}, \mathbf{A})} = \mathbf{A}^T\cdot\mathbf{X}$ be the summation aggregation function. 
	Assume that the input feature matrix $\mathbf{X}$ is row-normalized, such that $\forall v\in\mathcal{V}: \norm{\mathbf{X}_v}= 1$.
	Then, the edge-level sensitivity of the aggregation function is $\Delta_{\AGG} = 1$.
\end{lemma}

\begin{proof}
	Let $\mathbf{A}$ and $\mathbf{A}^\prime$ be the adjacency matrices of two arbitrary edge-level adjacent graphs. 
	Therefore, there exist two nodes $u$ and $v$ such that:
	\begin{equation}
		\begin{cases}
			\mathbf{A}^\prime_{i,j}\neq\mathbf{A}_{i,j}, \quad \text{ if } i = u \text{ and } j = v,\\
			\mathbf{A}^\prime_{i,j}=\mathbf{A}_{i,j}, \quad  \text{ otherwise.}  \\
		\end{cases}
	\end{equation}
	Without loss of generality, we can assume that $\mathbf{A}_{v,u} = 1$ and $\mathbf{A}^\prime_{v,u} = 0$.
	The goal is to bound the following quantity:
	\[\norm[F]{\AGG{(\mathbf{X}, \mathbf{A})} - \AGG{(\mathbf{X}, \mathbf{A}^\prime)}}.\]
	Let $\mathbf{M} = \AGG{(\mathbf{X}, \mathbf{A})}$ be the aggregation function output on $\mathbf{A}$, and 
	\[\mathbf{M}_i=\sum_{j=1}^{N}\mathbf{A}_{j,i}\mathbf{X}_j,\] 
	be the $i\text{-th}$ row of $\mathbf{M}$ corresponding to the aggregated vector for the $i\text{-th}$ node. Analogously, let $\mathbf{M}^\prime = \AGG{(\mathbf{X}, \mathbf{A}^\prime)}$. Then:
	\begin{align*}
		\norm[F]{\AGG{(\mathbf{X}, \mathbf{A})} &- \AGG{(\mathbf{X}, \mathbf{A}^\prime)}} 
		= \norm[F]{\mathbf{M}-\mathbf{M}^\prime} \\
		&= ({\sum_{i=1}^{N}\norm{\mathbf{M}_i-\mathbf{M}_i^\prime}^2})^{\nicefrac{1}{2}} \\
		&= \left({\sum_{i=1}^{N}\norm{\sum_{j=1}^{N}(\mathbf{A}_{j,i}\mathbf{X}_j-\mathbf{A}^\prime_{j,i}\mathbf{X}_j)}^2}\right)^{\nicefrac{1}{2}} \\
		&= \left({\norm{\mathbf{A}_{v,u}\mathbf{X}_v-\mathbf{A}^\prime_{v,u}\mathbf{X}_v}^2}\right)^{\nicefrac{1}{2}}\\
		&= \norm{(\mathbf{A}_{v,u}-\mathbf{A}^\prime_{v,u})\mathbf{X}_v} \\
		&= \norm{\mathbf{X}_v} \\
		&= 1,
	\end{align*}
	which concludes the proof.
\end{proof}

We can now prove \autoref{thm:privacy}.

\begin{proof}
	The PMA mechanism applies the Gaussian mechanism on the output of the summation aggregation function $\AGG{(\mathbf{X}, \mathbf{A})} = \mathbf{A}^T\cdot\mathbf{X}$. Based on \autoref{lem:e-sensitivity}, the edge-level sensitivity of $\AGG{(\cdot)}$ is 1. Therefore, according to Corollary~3 of~\cite{mironov2017renyi}, each individual application of the Gaussian mechanism is $(\alpha, \nicefrac{\alpha}{2\sigma^2})\text{-RDP}$. As PMA can be seen as an adaptive composition of $K$ such mechanisms, based on Proposition~1 of~\cite{mironov2017renyi}, the total privacy cost is $(\alpha, \nicefrac{K\alpha}{2\sigma^2})\text{-RDP}$.
\end{proof}

\subsection{Proof of \autoref{prop:gap-edge}}

\begin{proof}
	Under edge-level DP, only the adjacency information is protected. In \autoref{alg:gap}, the only step where the graph's adjacency is used is the application of the PMA mechanism (step 4), which according to \autoref{thm:privacy} is $(\alpha, \nicefrac{K\alpha}{2\sigma^2})\text{-RDP}$. 
	Since EM does not use the graph's edges and the classification module only post-process the private aggregated features without accessing the edges again, the total privacy cost remains $(\alpha, \nicefrac{K\alpha}{2\sigma^2})\text{-RDP}$. Therefore, according to \autoref{prop:rdptodp} it is equivalent to edge-level $(\epsilon, \delta)$-DP with $\epsilon=\frac{K\alpha}{2\sigma^2} + \frac{\log(1/\delta)}{\alpha-1}$. Minimizing this expression over $\alpha>1$ 
	gives $\epsilon = \frac{K}{2\sigma^2} + \nicefrac{\sqrt{2K\log{(1/\delta)}}}{\sigma}$.
\end{proof}

\subsection{Proof of \autoref{thm:nodeprivacy}}
We first prove \autoref{lem:n-sensitivity-small} and \autoref{prop:agg-rdp}, and then prove \autoref{thm:nodeprivacy}.

\begin{lemma}\label{lem:n-sensitivity-small}
    Given any graph $\mathcal{G}=(\mathcal{V}, \mathcal{E}, \mathbf{X})$, let 
    \[\agg \left(\{\mathbf{X}_u: \forall u \in \N_v\}\right) = \sum_{u\in\N_v}\mathbf{X}_u\]
    be the summation aggregation function over the neighborhood $\N_v$ of any arbitrary node $v\in\mathcal{V}$. 
	Assume that the input feature matrix $\mathbf{X}$ is row-normalized, such that $\forall v\in\mathcal{V}: \norm{\mathbf{X}_v}= 1$. 
	Then, the node-level sensitivity of $\agg(.)$ is $\Delta_{\agg} = 1$.
\end{lemma}

\begin{proof}
	Consider a node-level adjacent graph $\mathcal{G}^\prime=(\mathcal{V}^\prime, \mathcal{E}^\prime, \mathbf{X}^\prime)$ formed by adding a single node $q$ to $\mathcal{G}$. Hence, we have  $\mathcal{V}^\prime = \mathcal{V}\cup\{q\}$, and $\mathbf{X}^\prime_v = \mathbf{X}_v$ for every node $v\in\mathcal{V}$. Let $\mathbf{A}$ and $\mathbf{A}^\prime$ be the adjacency matrices of $\mathcal{G}$ and $\mathcal{G}^\prime$ respectively.
	The goal is to bound the following:
	\begin{equation}\label{eq:node-sens-agg}
		\norm[2]{\agg \left(\{\mathbf{X}_u: \forall u \in \N_v\}\right) - \agg \left(\{\mathbf{X}^\prime_u: \forall u \in \N^\prime_v\}\right)} \le 1.
	\end{equation}
	where $\N_v = \{u: \mathbf{A}_{u,v} = 1 \}$ and $\N^\prime_v = \{u: \mathbf{A}^\prime_{u,v} = 1\}$ are the adjacent nodes to $v$ in $\mathcal{G}$ and $\mathcal{G}^\prime$, respectively. Fixing any arbitrary node $v\in\mathcal{V}$, we have the following two cases:
	\begin{enumerate}
		\item If $q\in\N^\prime_v$, then we have $\N_v = \N^\prime_v \setminus \{q\}$. Therefore:
		\begin{align*}
			\norm[2]{\agg \left(\{\mathbf{X}_u: \forall u \in \N_v\}\right) &- \agg \left(\{\mathbf{X}^\prime_u: \forall u \in \N^\prime_v\}\right)} \\
			&= \norm[2]{\sum_{u\in\N_v}\mathbf{X}_u - \sum_{u\in\N^\prime_v}\mathbf{X}^\prime_u} \\
			& = \norm[2]{\mathbf{X}_q} = 1.
		\end{align*}

		\item If $q\notin\N^\prime_v$, then we have $\N_v = \N^\prime_v$. Therefore:
		\begin{align*}
			\norm[2]{\agg \left(\{\mathbf{X}_u: \forall u \in \N_v\}\right) &- \agg \left(\{\mathbf{X}^\prime_u: \forall u \in \N^\prime_v\}\right)} \\
			&= \norm[2]{\sum_{u\in\N_v}\mathbf{X}_u - \sum_{u\in\N^\prime_v}\mathbf{X}^\prime_u} = 0.
		\end{align*}
	\end{enumerate}
	\autoref{eq:node-sens-agg} follows from the above two cases.
\end{proof}

\begin{lemma}\label{prop:agg-rdp}
    Given any graph $\mathcal{G}=(\mathcal{V}, \mathcal{E}, \mathbf{X})$ with adjacency matrix $\mathbf{A}$ and maximum degree bounded above by some constant $D>0$, assume that the feature matrix $\mathbf{X}$ is row-normalized, such that $\forall v\in\mathcal{V}: \norm{\mathbf{X}_v}= 1$.
    Let $\agg \left(\{\mathbf{X}_u: \forall u \in \N_v\}\right) = \sum_{u\in\N_v}\mathbf{X}_u$ be the summation aggregation function over the neighborhood $\N_v$ of any arbitrary node $v\in\mathcal{V}$, and
	$\widetilde{\AGG}{(\mathbf{X}, \mathbf{A})}$ be a noisy aggregation mechanism which applies the Gaussian mechanism independently on the aggregated vector of every individual node as:
	\begin{equation*}
	    \widetilde{\AGG}{(\mathbf{X}, \mathbf{A})} = \left[ \agg \left(\{\mathbf{X}_u: \forall u \in \N_v\}\right) + \mathcal{N}(\sigma^2\mathbb{I}) : \forall v\in\mathcal{V} \right]^T.
	\end{equation*} 
	 Then $\widetilde{\AGG}(.)$ is $(\alpha, \nicefrac{D\alpha}{2\sigma^2})\text{-RDP}$.
\end{lemma}

\begin{proof}
    According to~\autoref{lem:n-sensitivity-small}, the node-level sensitivity of $\agg \left(\{\mathbf{X}_u: \forall u \in \N_v\}\right)$ is 1, and thus each individual noisy aggregation query is $(\alpha, \nicefrac{\alpha}{2\sigma^2})\text{-RDP}$. Although $\widetilde{\AGG}$ is composed of $N=\lvert\mathcal{V}\rvert$ such queries in total (one noisy aggregation per node), as $\mathcal{G}$'s maximum degree is bounded above by $D$, the embedding $\mathbf{X}_u$ of each node $u$ only contributes to maximum $D$ out of $N$ queries. As these $N$ queries are chosen non-adaptively and the noise of the Gaussian mechanism is independently drawn for each query, the maximum privacy cost of $\widetilde{\AGG}(.)$ is equivalent to $D$ compositions of $(\alpha, \nicefrac{\alpha}{2\sigma^2})\text{-RDP}$ mechanisms, which based on Proposition~1 of~\cite{mironov2017renyi} is $(D\alpha, \nicefrac{\alpha}{2\sigma^2})\text{-RDP}$.
\end{proof}

Now, we prove \autoref{thm:nodeprivacy}.

\begin{proof}
	At each step of the PMA mechanism, the Gaussian mechanism is applied on every output row of the summation aggregation function $\AGG{(\mathbf{X}, \mathbf{A})} = \mathbf{A}^T\cdot\mathbf{X}$. Based on \autoref{prop:agg-rdp}, this mechanism is $(\alpha, \nicefrac{\alpha D}{2\sigma^2})\text{-RDP}$. As PMA can be seen as an adaptive composition of $K$ such mechanisms, based on Proposition~1 of~\cite{mironov2017renyi}, the total privacy cost is $(\alpha, \nicefrac{\alpha DK}{2\sigma^2})\text{-RDP}$.
\end{proof}


\subsection{Proof of \autoref{prop:gap-node}}

\begin{proof}
	Under node-level DP, all the information pertaining to an individual node, including its features, label, and edges, are private. The first step of \autoref{alg:gap} privately processes the node features and labels so as to satisfy $(\alpha, \epsilon_1(\alpha))$-RDP. Steps 2 and 3 of the algorithm, however, expose the private node features, but then they are processed by steps 4 and 5, which are $(\alpha, \nicefrac{DK\alpha}{2\sigma^2})\text{-RDP}$ (according to \autoref{thm:nodeprivacy}) and $(\alpha, \epsilon_5(\alpha))$-RDP, respectively. As a result, \autoref{alg:gap} can be seen as an adaptive composition of an $(\alpha, \epsilon_1(\alpha))$-RDP mechanism, an $(\alpha, \nicefrac{DK\alpha}{2\sigma^2})\text{-RDP}$ mechanism, and an $(\alpha, \epsilon_5(\alpha))$-RDP mechanism.
	Therefore, based on Proposition~1 of~\cite{mironov2017renyi}, the total node-level privacy cost of \autoref{alg:gap} is $(\alpha, \epsilon_1(\alpha)+\nicefrac{DK\alpha}{2\sigma^2}+\epsilon_5(\alpha))\text{-RDP}$, which ensures $(\epsilon_1(\alpha) + \epsilon_5(\alpha) + \frac{DK\alpha}{2\sigma^2} + \frac{\log(1/\delta)}{\alpha-1}, \delta)$-DP based on \autoref{prop:rdptodp}.
\end{proof}

\end{document}